\documentclass[12pt]{article}
\usepackage[margin=1in]{geometry}
\pdfoutput=1 

\usepackage{microtype}
\usepackage{graphicx}
\usepackage{subfigure}
\usepackage{enumitem}
\usepackage{bbm}
\usepackage{booktabs} 

\usepackage{hyperref}

\newcommand{\qx}[1]{\textcolor{cyan}{(QX: #1)}}

\newcommand{\bitem}{\begin{itemize}}
\newcommand{\eitem}{\end{itemize}}
\newcommand{\benum}{\begin{enumerate}}
\newcommand{\eenum}{\end{enumerate}}
\newcommand{\bdefn}{\begin{definition}}
\newcommand{\edefn}{\end{definition}}
\newcommand{\bprop}{\begin{proposition}}
\newcommand{\eprop}{\end{proposition}}
\newcommand{\bque}{\begin{question}}
\newcommand{\eque}{\end{question}}
\newcommand{\bobsv}{\begin{observation}}
\newcommand{\eobsv}{\end{observation}}
\newcommand{\beqn}{\begin{equation}\begin{aligned}}
\newcommand{\eeqn}{\end{aligned}\end{equation}}

\newcommand{\ps}{\begin{proof}[Sketch]}
\newcommand{\brmk}{\begin{remark}}
\newcommand{\ermk}{\end{remark}}
\newcommand{\bduiqi}{\begin{aligned}}
\newcommand{\eduiqi}{\end{aligned}}
\newcommand{\bcoro}{\begin{corollary}}
\newcommand{\ecoro}{\end{corollary}}
\newcommand{\bcom}{}

\newcommand{\avg}{average}
\newcommand{\arb}{arbitrary}

\newcommand{\alg}{algorithm}

\newcommand{\Alg}{Algorithm}

\newcommand{\assu}{assumption}

\newcommand{\bs}{\backslash}

\newcommand{\cond}{condition}

\newcommand{\corres}{corresponding}
\newcommand{\conti}{continuous}

\newcommand{\ci}{confidence interval}

\newcommand{\distr}{distribution}

\newcommand{\func}{function}

\newcommand{\ho}{\mathbb}

\newcommand{\indep}{independent}

\newcommand{\IOW}{In other words}

\newcommand{\ih}{induction hypothesis}

\newcommand{\ins}{instance}

\newcommand{\info}{information}

\newcommand{\ineq}{inequality}

\newcommand{\IFT}{It follows that}

\newcommand{\ift}{it follows that}

\newcommand{\kehua}{characterize}

\newcommand{\lar}{\leftarrow}

\newcommand{\mtg}{martingale}

\newcommand{\Ow}{Otherwise}
\newcommand{\ow}{otherwise}
\newcommand{\omg}{\omega}
\newcommand{\Omg}{\Omega}

\newcommand{\OTOH}{On the other hand}

\newcommand{\rb}{\right}
\newcommand{\lb}{\left}

\newcommand{\prb}{probability}

\newcommand{\parti}{particular}

\newcommand{\pmt}{parameter}

\newcommand{\rv}{random variable}

\newcommand{\rar}{\rightarrow}

\newcommand{\real}{\mathbb{R}}
\newcommand{\resp}{respectively}

\newcommand{\sats}{satisfies}

\newcommand{\sse}{\subseteq}
\newcommand{\sps}{suppose}
\newcommand{\Sps}{Suppose}

\newcommand{\strfwd}{straightforward}

\newcommand{\sym}{symmetry}

\newcommand{\xulie}{sequence}

\let\eps\varepsilon

\usepackage{amsmath}
\usepackage{amssymb}
\usepackage{mathtools}
\usepackage{amsthm}
\usepackage{lipsum}

\usepackage[capitalize,noabbrev]{cleveref}

\theoremstyle{plain}
\newtheorem{theorem}{Theorem}[section]
\newtheorem{proposition}[theorem]{Proposition}
\newtheorem{lemma}[theorem]{Lemma}
\newtheorem{corollary}[theorem]{Corollary}
\theoremstyle{definition}
\newtheorem{definition}[theorem]{Definition}

\theoremstyle{remark}
\newtheorem{remark}[theorem]{Remark}

\usepackage[textsize=tiny]{todonotes}

\usepackage{natbib}
\usepackage{algorithm,algorithmicx}
\usepackage[noend]{algpseudocode}
\usepackage{authblk}

\title{\LARGE \bf Smooth Non-Stationary Bandits}
\author{Su Jia, Qian Xie, Nathan Kallus, Peter I. Frazier}
\affil{Cornell University}

\begin{document}
\maketitle

\begin{abstract}
In many applications of online decision making, the environment is non-stationary and it is therefore crucial to use bandit algorithms that handle changes. 
Most existing approaches are designed to protect against non-smooth changes, constrained only by total variation or Lipschitzness over time. 
However, in practice, environments often change {\em smoothly}, so such algorithms may incur higher-than-necessary regret.
We study a non-stationary bandits problem where each arm's mean reward sequence can be embedded into a $\beta$-H\"older function, i.e., a function that is $(\beta-1)$-times Lipschitz-continuously differentiable.
The non-stationarity becomes more smooth as $\beta$ increases. 
When $\beta=1$, this corresponds to the non-smooth regime, where \cite{besbes2014stochastic} established a minimax regret of $\tilde \Theta(T^{2/3})$.
We show the first separation between the smooth (i.e., $\beta\ge 2$) and non-smooth (i.e., $\beta=1$) regimes by presenting a policy with $\tilde O(k^{4/5} T^{3/5})$ regret on any $k$-armed, $2$-H\"older instance.
We complement this result by showing that the minimax regret on the $\beta$-H\"older family of instances is
$\Omg(T^{(\beta+1)/(2\beta+1)})$ for any integer $\beta\ge 1$.
This matches our upper bound for $\beta=2$ up to logarithmic factors.
Furthermore, we validated the effectiveness of our policy through a comprehensive numerical study using real-world click-through rate data.
\end{abstract}


\section{Introduction}\label{sec:intro}
In numerous decision-making problems characterized by uncertainty and limited feedback on the outcomes, an agent must judiciously leverage past observations to make informed future decisions.

The {\em Multi-armed Bandits} (MAB) framework provides a paradigm for addressing this issue, as it encapsulates the fundamental trade-off between (i) exploration, i.e., acquiring new information, and (ii) exploitation, i.e., generating rewards based on available information.

Further complicating this challenge is the temporal evolution of the underlying environment, a characteristic prevalent in many real-world problems in operations and marketing.
This motivates the study of the non-stationary bandits problem. 
In the standard model, the adversary is limited by how much the mean reward is allowed to change {\em in total} over time \citep{besbes2014stochastic}.
Formally, the {\it total variation} is defined as $\sum_{t=1}^T |r_a(t) - r_a(t+1)|$, where $r_a(t)$ is the mean reward of arm $a$ in round $t$, and $T$ is the total number of rounds.

The problem within this framework has been extensively studied. \cite{besbes2014stochastic} introduced a policy that periodically ``restarts'' the learning process by discarding all prior observations.
With total variation $V$, this policy achieves $\tilde O\lb (V^{1/3} T^{2/3}\rb)$ regret (i.e., the loss relative to the optimal policy if the mean rewards were known, as formally defined in Section \ref{sec:formulation}); see their Theorem 2.
Moreover, this result is best possible.  
For any $V$, any policy suffers $\Omg\lb (V^{1/3} T^{2/3}\rb)$ regret on some instance with total variation $V$; see their Theorem 1.

However, the total variation model may be too pessimistic, as it allows the adversary to shock the mean reward {\bf instantaneously}. 
Can we achieve lower regret by requiring the mean rewards to change {\bf continuously}? 
More precisely, suppose that the sequence  $\{r_a(t)\}_{t\in [T]}$ of mean rewards  can be {\em embedded} into an $L$-Lipschitz function, in the sense that there exists an $L$-Lipschitz \func\ $f_a : [0,1]\rar \real$ such that for all $t\in [T]$ we have
\begin{align}\label{eqn:embed}
r_a(t) = f_a\lb(\frac tT\rb).
\end{align}
Note that by Lipschitzness of $f_a$, we have
$V\le L$, so we immediately obtain an $\tilde O(L^{1/3} T^{2/3})$ regret bound. 
The above question can then be formalized as: 
\begin{center}
Q1: {\em Can we improve this $\tilde O(T^{2/3})$ bound by capitalizing on Lipschitz continuity?}
\end{center} 

The answer is negative. In fact, although \cite{besbes2014stochastic} proved the $\Omg\lb(V^{1/3} T^{2/3}\rb)$ lower bound using {\em discontinuous} reward functions, an alternate proof can be given using Lipschitz continuous reward functions; see our Theorem \ref{thm:lb}.
In other words, the impact of continuity is negligible compared to the constraints imposed by the total variation budget.

But there is more to the story. The Lipschitz non-stationarity model may still be overly pessimistic, since it allows the adversary to instantly shock the {\bf rate of change} in the mean reward (formally, the derivative of $f_a$ in eqn. \eqref{eqn:embed}).
However, in many applications, the underlying environment changes in a {\em smooth} manner, such as gradual changes in temperature, economic factors, and demand for seasonal products, just to name a few.
Furthermore, the capacity to abruptly change slope plays a pivotal role in the lower-bound proof under the Lipschitz model. 
This suggests that the optimal regret bound may be improved with a suitable smoothness assumption.

To formalize the notion of smooth non-stationarity, we borrow a standard concept, the {\em H\"older class}, from nonparametric statistics.
For any integer $\beta \ge 1$, a \func\ is  {\it $\beta$-H\"older} if the first $(\beta-1)$ derivatives exist and are Lipschitz (the formal definition is deferred to Section \ref{sec:formulation}). 
An instance is $\beta$-H\"older if the reward function $r_a(\cdot)$ of each arm can be embedded into a $\beta$-H\"older function $f_a(\cdot)$ on $[0,1]$, as formally specified by eqn. \eqref{eqn:embed}.

In \parti, for $\beta=1$, this recovers the Lipschitz model discussed above.
Setting $\beta \ge 2$, we constrain the adversary more than in the previous literature, since the sequence of mean rewards can be embedded into a {\bf differentiable} function.
This motivates the central question of this work: 
\begin{center} Q2: \it Can we break the minimax regret bound, $\tilde \Theta(T^{2/3})$, for the non-smooth setting (i.e., $\beta=1$), under smooth non-stationarity (i.e., $\beta \ge 2$)?
\end{center}

We provide an affirmative answer by showing an $ \tilde O(T^{3/5})$ upper bound on the minimax regret for $\beta=2$.
Moreover, this bound is {\bf nearly optimal}: We show that for any integer $\beta\ge 1$, the minimax regret is $\Omg(T^{(\beta+1)/(2\beta+1)})$ for $\beta$-H\"older instances. 
In \parti, for $\beta=2$, our upper and lower bounds match up to logarithmic factors.

\subsection{Our Contributions} 
Our contributions can be categorized into the following three parts. 

\benum 
\item {\bf A novel formulation.} Our first contribution is to introduce {\em smoothness} in non-stationary online models.
We quantify the smoothness of an arrival sequence using a standard concept, {\em H\"older smoothness}, from non-parametric statistics.
Loosely, a non-stationary bandits instance is $\beta$-H\"older if the mean reward sequence can be embedded into a $\beta$-H\"older function (formally defined in Section \ref{sec:formulation}).
Our formulation provides a unified framework for existing non-stationary models: Special cases have been studied for $\beta =1$ \citep{besbes2014stochastic} and  $\beta\in (0,1)$ \citep{manegueu2021generalized}.
\item {\bf Upper bounds.} We present a simple policy and analyze its regret for the settings $\beta = 1,2$. 
Specifically, we show the following results.
\benum
\item {\bf A simple derivative-free policy:} We propose a simple  policy, dubbed the {\em Budgeted Exploration} (BE) policy, which ``restarts'' periodically by discarding all historical observations.
Surprisingly, our policy does {\em not} use any derivative information.
\item {\bf Upper bound for $\beta =1$: an alternate proof.} 
Theorem 1 in \citealt{besbes2014stochastic} implies that their policy (which is different from ours) has $\tilde O(T^{2/3})$ regret for $\beta = 1$. 
We show that our BE policy has $\tilde O(T^{2/3})$ regret; see our Theorem \ref{thm:be_beta1}.
\item {\bf Regret analysis via a potential function argument:} We analyze the regret of the BE policy using a {\it potential \func} argument, a common technique in competitive analysis but less so in MAB. 
Informally, consider an adversary who aims to produce a mean-reward sequence to fool the learner and hence forge a high regret. 
We argue that if the regret is high in some interval, then the adversary loses a {\em proportionally} large amount of energy (in terms of $T$) and is therefore less powerful in the future.
Therefore, we can bound the total regret in terms of the initial energy of the adversary. 
\item {\bf Separation between smooth and non-smooth settings:} We show that our BE policy has $\tilde O(k^{4/5} T^{3/5})$ regret when $\beta=2$; see Theorem \ref{thm:ub_k_arms}.
We emphasize that this is the first {\em separation} between the smooth (i.e., $\beta\ge2$) and non-smooth (i.e., $\beta=1$) regimes for non-stationary bandits: 
For $\beta=1$, the minimax regret is $\Omg(T^{2/3})$ (see (b)), which is higher than our upper bound for $\beta =2$. 
\eenum 
\item {\bf Lower bounds.}  We provide an $\Omg(T^{(\beta+1)/(2\beta+1)})$ lower bound on the minimax regret that is valid for any integer $\beta\ge 1$; see Theorem \ref{thm:lb}. 
This result holds significance due to the following.
\benum
\item {\bf Minimax optimality.} For $\beta=2$, this lower bound matches our $\tilde O(T^{3/5})$ upper bound for $\beta =2$ up to logarithmic factors.
\item {\bf Lower bound for $\beta =1$: an alternative proof.} \cite{besbes2014stochastic} showed that for any total variation budget $V$, any policy has $\Omg(V^{1/3}T^{2/3})$ regret; see their Theorem 1. 
Our Theorem \ref{thm:lb} implies an alternative proof for this result by setting $\beta =1$. 
Moreover, the construction in \citealt{besbes2014stochastic} employs mean-reward sequences with ``jumps''.
In contrast, our proof uses mean-reward \xulie s that change continuously.
This suggests a surprising insight: Unlike between the smooth and non-smooth settings, there is {\em no} separation between the continuous and discontinuous settings. 
\item {\bf Conjecture: smooth non-stationary bandits is almost as easy as stationary bandits?} 
Our general lower bound leads us to the conjecture that the lower bounds can be matched for $\beta\ge 3$. 
Assuming it is true, then the minimax regret $T^{1/2 + O(1/\beta)}$ tends to $\sqrt T$, as $\beta\rar \infty$, almost matching that of {\em stationary} bandits. 
In other words, this would lead to a rather surprising insight: Non-stationary bandits with high smoothness is almost as manageable as stationary bandits.
\eenum
\eenum 

\section{Formulation}\label{sec:formulation}
We now formally define our non-stationary bandit model.
We are given a finite set $[k]$ of arms and a known finite time horizon $[T]$. (For any integer $n\ge 1$ we write $[n] := \{1,\dots,n\}$.) 
For each arm $a\in [k]$, let $(Z_a^t)_{t\in [T]}$ be independent $\sigma$-subgaussian random variables that represent the realized rewards. 
For simplicity, we assume $\sigma=1$; our results can be extended to any $\sigma>0$ \strfwd ly.
With this simplification, we can fully specify an {\it instance} by the {\it reward \func} $r_a(t) := \ho{E}[Z_a^t]$ of the arms.
As is standard in the prevailing literature, we assume that the mean rewards are bounded. 
Without loss of generality, we assume that $\ho{E}[Z_a^t] \in [-1,1]$ for all $a\in [k]$ and $t\in [T]$.

In each round $t\in [T]$, the learner selects an arm $A_t$ and receives an observable {\it reward} $Y_t := Z_{A_t}^t$.
Based on the past rewards observed, the learner selects the next arm $A_{t+1}$, with the goal of maximizing the cumulative rewards.
This decision-making process is called a {\it policy}. Formally, a policy $(A_t)_{t\in [T]}$ is a stochastic process such that $A_t$ is $\mathcal{F}_{t-1}$-measurable for each $t\in [T]$, where $\mathcal{F}_{t-1}$ is the $\sigma$-algebra generated by $\{Z_a^s: a\in [k], 1\le s\le t-1\}$.

An important special case is the {\em one-armed} setting, which is often used to highlight the essential ideas in non-stationary bandits (see, e.g., Section 35.3 of \citealt{lattimore2020bandit}).
There are two arms: a {\em static arm}, whose reward function is a constant, $0$, and a {\em changing arm}.  
The core challenge is to learn the sign of the mean reward of the changing arm. 

\subsection{The H\"older Class and Smooth Non-stationary Bandits}
For every arm $a$, the reward function $r_a(t)$ is specified over the {\bf discrete} domain $\{1,\dots,T\}$, making it unclear how to characterize the smoothness of non-stationarity in this context.
To do this, we first define the smoothness of a \func\ defined on a {\bf continuous} domain.

\begin{definition}[H\"older Class] For any $L>0$ and integer $\beta \ge 1$, we say a \func\ $f:[0,1]\rar \real$ is {\it $\beta$-H\"older} and write $f \in \Sigma(\beta, L)$ if (i) $f$ is ($\beta-1$)-times differentiable, and  (ii) $f^{(\beta-1)}$ and $f$ are $L$-Lipschitz. \end{definition}

For concreteness, observe that $f\in \Sigma(1,L)$ if and only if $f$ is $L$-Lipschitz, and that $f\in \Sigma(2,L)$ if and only if $f$ is differentiable and $f',f$ are $L$-Lipschitz. 

A bandit instance is $\beta$-H\"older if its reward \func\ can be {\em embedded} into a $\beta$-H\"older function in the following sense.

\begin{definition}[Smooth Non-stationary Instance]
\label{def:Holder_ins} We say that a sequence $\{r(t)\}$ can be {\em embedded} in a function $\mu:[0,1]\rar \real$ if for any $t\in [T]$, it holds that \[r(t) = \mu\lb(\frac tT\rb).\]
A non-stationary bandit \ins\ $r = \{r_a(t)\}$ is {\it $\beta$-H\"older}, if the reward function of every arm $a\in [k]$ can be embedded in some \func\ $\mu_a\in \Sigma(\beta, L)$.
\end{definition}

For example, in the one-armed case, consider $r_1(t)=f(t/T)$ for $f(x)= |x-\frac 12|$. 
This \ins\ is $1$-H\"older but {\bf not} $2$-H\"older. 
In fact, for any differentiable \func\ $g(x)$ with $r_1(t) = g(t/T)$, the derivative $g'(x)$ changes from $-1$ to $+1$ in an interval of length $O(1/T)$ near $x=\frac 12$. 
Consequently, $|g''(x)| = \Omg(T)$, but the definition of $2$-H\"older requires $g''$ be bounded by an {\em absolute} constant.

We emphasize that in Definition \ref{def:Holder_ins}, $\mu_a$ is defined on the {\it normalized} time scale $[0,1]$ while $r_a$ is defined on $[T] = \{1,\dots,T\}$.
To avoid confusion, we will use ``$x$'' for the $[0,1]$ scale and ``$t$'' for the $[T]$ scale. 
This work focuses on characterizing the minimax regret of the H\"older family, formally defined as follows.

\begin{definition}[The H\"older Family]
For any integer $k\ge 2$, $\beta \ge 1$, we denote by $\Sigma_k(\beta, L)$ the family of $k$-armed, $\beta$-H\"older \ins s.
\end{definition}

We use $\Sigma_1(\beta,L)$ to denote the family of one-armed instances (although these instances actually involve two arms).
Next, we will define the metric for  assessing a policy.

\subsection{The Regret} 
The problem is trivial if $r_a(\cdot)$'s were known. 
In each round $t$, the optimal policy chooses $a_t^*\in \arg\max_{a\in\{0,1\}} r_a(t)$ and collects, in expectation, a reward of $r^*(t) = \max_{a\in \{0,1\}} \{r_a(t)\}$. 
When $r_a(\cdot)$'s are unknown, the policy needs to learn the reward functions. 
Differently from stationary stochastic bandits or adversary bandits - where a policy is compared against the best fixed arm - the benchmark in the non-stationary setting may select different arms over time. 
The {\it dynamic regret} is the difference between this benchmark and the expected total rewards of the policy.

\begin{definition}[Regret]
The {\it dynamic regret} (or simply {\em regret}) of a policy $A$ for an instance $r=\{r_a(t)\}$ is 
$\textstyle\mathrm{Reg}(A, r) = \ho{E} \lb[\sum_{t=1}^T \lb(r^*(t) - Z_{A_t}^t\rb)\rb].$
For a family $\cal F$ of \ins s, the {\it worst-case regret} of $A$ is $\max_{r\in \cal F} \mathrm{Reg}(A, r)$. 
The {\it minimax regret} of $\cal F$ is the minimum achievable worst-case regret among all policies, formally defined as $\inf_A \sup_{r\in \cal F} {\rm Reg}(A,r)$, where ``inf'' is over all policies.
\end{definition}

Compared to static regret, dynamic regret better captures the loss due to not knowing the underlying reward function and therefore serves as a better performance metric compared to against the best fixed arm.
Moreover, it is easy to verify that any bound on the dynamic regret carries over to static regret, but not vice versa. 
Furthermore, in practical scenarios, such as those involving seasonal patterns, dynamic regret tends to be more meaningful than static regret.
For example, suppose that the reward function is given by $f(t) = \cos (2\pi t)$. 
Then, any fixed-arm policy has $0$ static regret and $\Omg(T)$ dynamic regret.

It is crucial for the reader not to confuse {\it H\"older} class with {\it H\"older continuity}. 
A function $f:[0,1]\rar \real$ is {\it $\alpha$-H\"older-continuous} if $|f(x) - f(x')|\le |x-x'|^\alpha$ for all $x,x'\in [0,1]$.
This is only meaningful for $\alpha\le 1$, as otherwise the function must be a constant.
In its most general form, for any $\beta>0$, the $\beta$-H\"older class consists of all functions $f:[0,1]\rar \real$ such that (i) $f$ is $\ell := \lfloor \beta \rfloor$ times differentiable and (ii) $f^{(\ell)}$ is $(\beta-\ell)$-H\"older-continuous; see, e.g., Definition 1.2 of \citealt{tsybakov2004introduction}.
To highlight the key ideas, we focus on integer $\beta$'s, although the results can be extended to non-integer $\beta$'s  \strfwd ly.

\cite{manegueu2021generalized} considered {\it H\"older-continuous} non-stationary  mean reward and considered embedding the mean reward sequence into a function defined on a continuous domain, as specified in our eqn. \eqref{eqn:embed}; see their Section 5.3.
Although they referred to their problem as ``switching smooth expected rewards'', the slope of the mean reward may change abruptly, and hence the instances they considered are not really smooth.

From a more unified point of view, \cite{manegueu2021generalized} considered $\beta<1$, \cite{besbes2014stochastic} considered $\beta = 1$, and we are the first to consider $\beta >1$, which is our main contribution.
The leap from $\beta=1$ to $\beta > 1$ is substantially {\bf more} challenging than from $\beta=1$ to $\beta < 1$. 
The former must leverage derivatives, whereas the latter can still work with the functions themselves while better tracking general $\beta < 1$ exponents in arguments similar to the case of $\beta=1$.

\section{Related Work}\label{sec:related_work}
We classify the related papers based on the following four characteristics: model-free non-stationarity, model-based non-stationarity, smoothness in bandits, and monotonicity in exploration.

\noindent{\bf Model-free Non-stationarity.}
Recent multi-armed bandit literature has recognized the importance of 
non-stationarity. 
We say that a model is {\em model-free} if the non-stationarity is not defined by a statistical temporal model.
\cite{auer2002nonstochastic} considered an extremely general model where the reward function can be {\em arbitrarily} non-stationary. 
However, their $\tilde O(\sqrt T)$ regret bound is not comparable to ours, since they compare against the best fixed arm in hindsight (i.e., static regret).
It should be noted that this frame can also capture a slightly more general benchmark called {\em tracking regret}, where we compare against the best policy that changes arms a certain number of times.
It can be shown that for any integer $m>0$, the EXP4 policy has $\tilde O(\sqrt{Tmk})$ regret against any policy that changes arm $m$ times; see eqn. (31.2) in \citealt{lattimore2020bandit}.

Variants of the UCB \alg\ have been investigated for non-stationary environments.
The sliding-window UCB \alg\ was introduced without analysis by \cite{kocsis2006bandit}.
\cite{garivier2011upper} analyzed this \alg\ assuming that the mean rewards remain constant during periods and change at an unknown time (``breakpoint'').
This approach has also been extended to linear bandits \citep{cheung2019learning} and MDP \citep{gajane2018sliding}.

Previous work has considered several aspects: contextual information \citep{luo2018efficient,russac2019weighted}; uncertainty in the number of changes
\citep{auer2019adaptively,chen2019new}; and Bayesian prior \info\ \citep{trovo2020sliding}, to name a few.
However, these papers make no smoothness assumptions on the non-stationarity.

In addition, the ``smoothly changing setting'' in \citet{trovo2020sliding} considers Lipschitz reward functions that may be {\em non-differentiable}, and is therefore incomparable with our work.
\cite{zhao2020dynamic} considered online convex optimization with ``smooth'' functions $f_t(x)$. 
They assumed that in each round the arriving loss function has a bounded gradient in the {\em action} $x$ (instead of in the time $t$), which is different from our focus. \\

\noindent{\bf Model-based Non-stationarity.}
There is a substantial literature that studies the dynamic aspect of the MAB problem by modeling the evolution of the reward function as a stochastic process. 
Motivated by task scheduling, \cite{gittins1979dynamic} considered a Bayesian setting in which the state of the selected arm evolves in a Markovian fashion and presented an optimal policy, based on the famous {\em Gittens index}.
Even more related is the {\em restless bandits} problem \citep{whittle1988restless}, where the state of each arm evolves regardless of whether it is selected, but we only observe the state and reward of the selected arm.  
This problem is rather challenging, since we need to learn the evolution of {\em all} the Markov chains. 

Several other papers explore more specialized stochastic models to account for non-stationarity.
\cite{slivkins2008adapting} assumed that the reward function is drawn from a known reflected Brownian motion with variance $\sigma^2$, and showed an optimal regret of $\tilde O\lb(k\sigma^2\rb)$.
Other evolution models include other Gaussian processes \citep{grunewalder2010regret} and discrete Markov chains \citep{zhou2021regime}. \\

\noindent{\bf Smoothness in Bandits.}
Model smoothness and the H\"older class are often studied in non-parametric statistics; see, e.g., \citealt{gyorfi2002distribution} and \citealt{tsybakov2004introduction}.
In the MAB literature, $\beta$-H\"older smoothness has been considered in contextual bandits, where  the mean reward is assumed to have a smooth dependence on the context. 
\cite{rigollet2010nonparametric} and \cite{perchet2013multi} primarily concentrated on cases where $\beta \le 1,$ while \cite{hu2020smooth} and \cite{gur2022smoothness} directed their attention towards scenarios where $\beta \ge 1.$

The key difference from our problem is {\em monotonicity} in the direction for exploration: Exploration can be done in any direction in the context space, but this is not true for the time axis.
Existing policies generally do not preserve monotonicity.
For example, Algorithm 1 in \citealt{hu2020smooth} maintains a hierarchical tree and zooms in on regions that are worth further exploration. 
This idea is apparently infeasible in our setting. \\

\noindent{\bf Monotonicity in Exploration.}
The {\em one-armed} version of our problem can be viewed as scanning the reward function monotonically at {\em uniform} speed, and making a decision on whether to pull arm $1$ in each round.
This is related to a line of work in which the learner is required to explore a domain {\em monotonically}, with the goal of stopping near the {\em mode} (i.e., maximum) of the unknown reward function \citep{chen2021multi,jia2021markdown,jia2022dynamic,salem2021algorithmic}. 

This problem can be viewed as the ``dual'' of our problem, where the learner controls the speed of scanning but must pull the changing arm in every round.
In this line of work, the impact of smoothness has been considered. 
The minimax regret when $\beta=1$ is $\tilde \Theta(T^{3/4})$ \citep{chen2021multi} and $\tilde \Theta(T^{5/7})$ when $\beta=2$ \citep{jia2021markdown}, under the extra assumption that the reward function is unimodal.
We reiterate that these results are not applicable to our problem, as we do not have control over the scanning speed. 

\section{Lower Bounds}\label{sec:lb}
In this section, we construct a $\beta$-H\"older family of one-armed instances for any given integer $\beta\ge 1$, and show that every policy has $\Omg(T^{(\beta+1)/(2\beta+1)})$ regret on it.
Each curve is defined as a {\em smooth} concatenation of several basic curves, in a way that prevents the learner from anticipating future evolution of the curve, even with access to past information.

\subsection{Building a Bowl}\label{sec:F_beta}
We will partition the normalized time horizon $[0,1]$ into {\em epochs} of length $\delta$.
Each curve in our family is a smooth concatenation of more basic functions, each being either a constant curve or a {\em bowl-shaped curve} (or a simple {\em bowl}).
A bowl consists of two {\em sides}, each of width $\Theta(\delta)$, and a flat part (i.e., constant) connecting them that has width $\Theta(\delta)$.
Specifically, each side of a bowl has {\em vanishing} derivatives up to order $(\beta-1)$ at the endpoints.
This property enables us to concatenate a bowl {\em smoothly} with another bowl or with the constant curve, satisfying the requirement for the H\"older class.
Furthermore, each bowl has a maximum height $\Omg(\delta^\beta)$. 
We will use this property to make a policy suffer $\Omg(\delta^\beta)$ regret on each epoch.
We state these properties formally below.

\begin{proposition}[Side of the Bowl]\label{prop:pyramid}
For any fixed integer $\beta\geq 1$, there exists a family $\{g_\eps\}$ of $(\beta-1)$-times continuously differentiable \func s where $g_\eps$ is defined on $[0,\eps]$, with\\
{\rm (i) vanishing derivatives:} $g^{(j)}_\eps(0)=g^{(j)}_\eps(\eps)=0$ for any $j=1,\dots,\beta-1$,\\
{\rm (ii) monotonicity:} $g'_\eps\geq 0$,\\
{\rm (iii) polynomial growth:}
$g_\eps (\eps) =\Theta(\eps^{\beta})$ as $\eps \rar 0^+$, or equivalently, $g_\eps (\eps)=C_\beta(\eps)\cdot \eps^{\beta}$ where $C_\beta(\eps)=\Theta(1)$, and
\\ 
{\rm (iv) Lipschitz derivatives:} $g^{(\beta-1)}_\eps$ is $1$-Lipschitz.
\end{proposition}

\begin{figure}
\centering \includegraphics[width=0.8\linewidth]{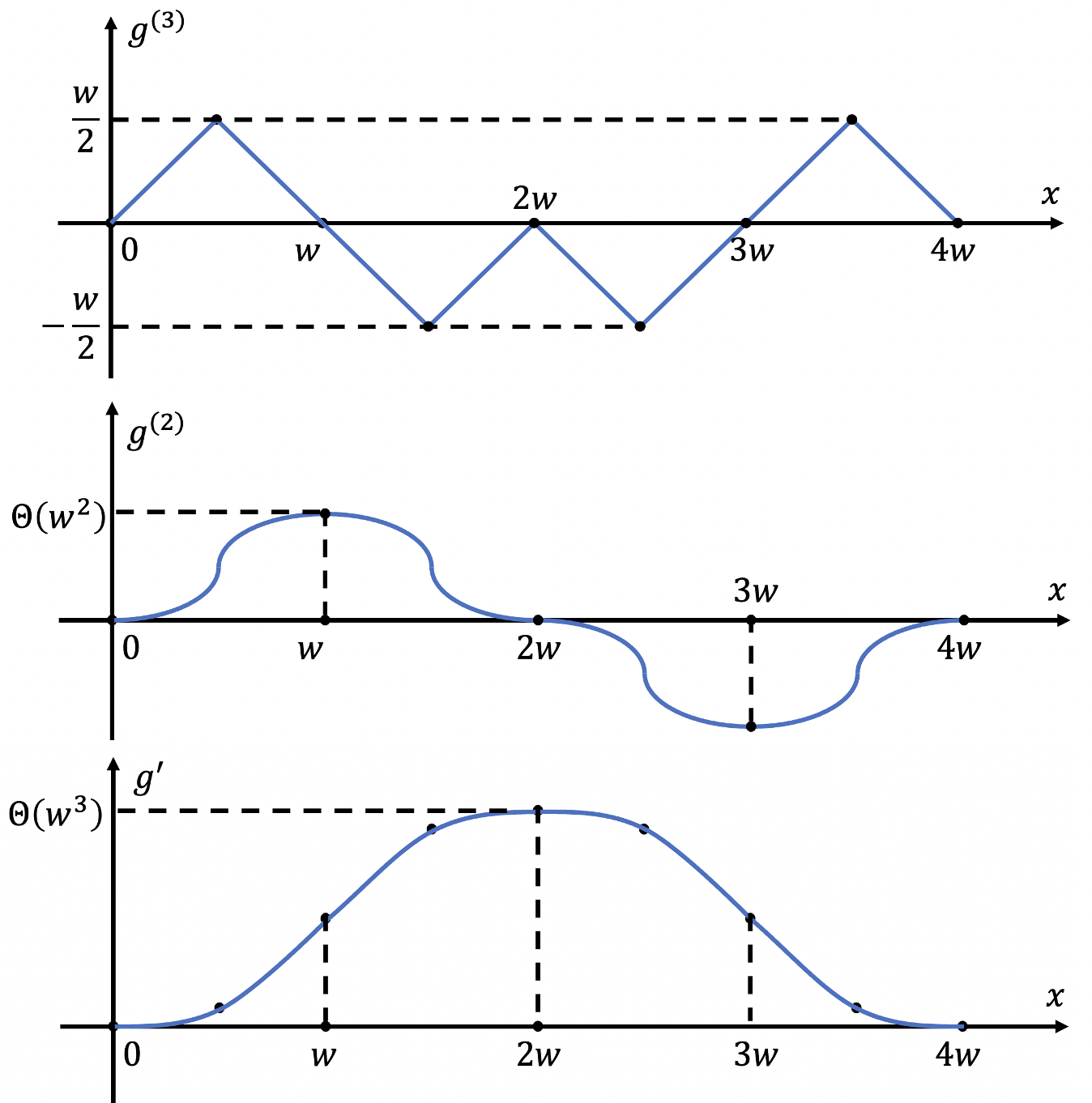}
\caption{Illustration of $g_\eps$ for $\beta = 4$: $g^{(3)}$ is a ``flock'' of pyramid-shaped \func s. The \func\ $g^{(2)}(x)$ is defined as the integration of $g^{(3)}$ from $0$ to $x$. 
Similarly, $g^{(1)}(x)$ is the integration of $g^{(2)}$ from $0$ to $x$. 
As the key property, any derivative \func\ lower than order $3$ vanishes at the boundary points, i.e., $0$ and $4w$.}
\label{fig:pyramid_flock_main_body}
\end{figure}

We illustrate the ideas in Figure \ref{fig:pyramid_flock_main_body}  for the special case of $\beta=4$.
Each pyramid-shaped function (formally defined in Section \ref{apdx:lb}) has width $w=\eps/4$, and slope $1$ and $-1$ on the two sides of its peak. 
Set the highest derivative, i.e., $g^{(3)}$, to be the concatenation of four pyramids, where the second and third pyramids are inverted around the $x$-axis; see the top-most subfigure.
Apparently, this function is $1$-Lipschitz, so (iv) holds.
Thus, if this \func\ is the third-order derivative of some function $g$, then $g$ is $3$-H\"older.
Such a \func\ $g$ can be constructed by repeatedly integrating the derivatives: For $i=3,2,1$, we iteratively define \[g^{(i-1)} (x) = \int_0^x g^{(i)}(s)\ ds\]
for each $x\in [0,\eps]$, and visualize $g^{(2)}, g^{(1)}$ in the other two subfigures in Figure \ref{fig:pyramid_flock_main_body}.

To complete the proof, we verify properties (i) to (iii) as follows.
\benum
\item[(i):] It is \strfwd\ to verify that $g^{(2)}(x),g^{(1)}(x),g^{(0)}(x)$ are all $0$ at the boundary points (i.e., $x=0$ and $x=4w$) as illustrated in the bottom subfigure.
Essentially, this follows from the \sym\ of the derivative \func s around the $x$-axis, i.e., the areas above and below $0$ are equal.
\item[(ii):] Observe in the bottom subfigure that $g'$ is nonnegative, so (ii) also holds.
\item[(iii):] Note that $g^{(3)}$ is $1$-Lipschitz and \sats\ $g^{(3)}(0) = 0$, so $g^{(3)}(\eps) = O(\eps)$.
It is then \strfwd\ to verify that $g^{(i)}(\eps) = O(\eps^{4-i})$ for any $0\le i\le 3$ and $\eps>0$. 
In \parti, we have $g(\eps)=O(\eps^4)$.
\eenum

We defer the details to Section \ref{apdx:lb}. In the next subsection, we will explain how to combine these bowls to construct a family of smooth \func s.

\subsection{Definition of the Family $\cal F_\beta$}

\begin{figure}
    \centering    \includegraphics[width=0.9\linewidth]{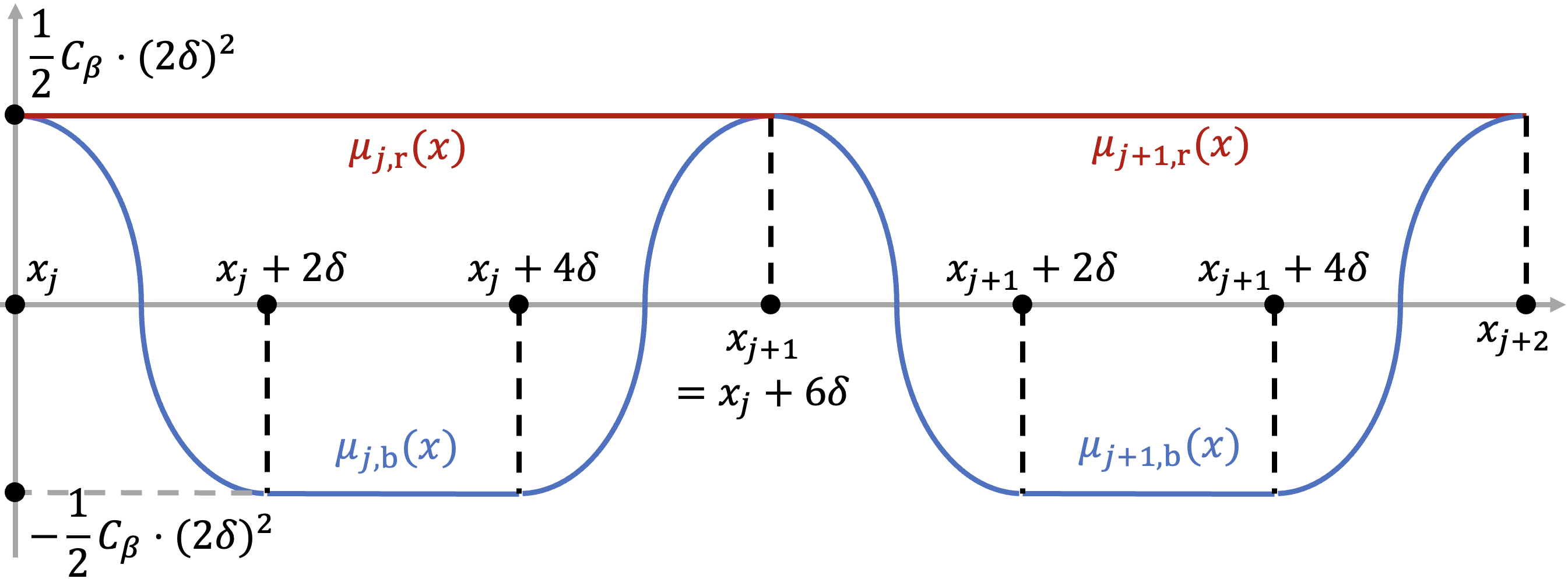}\caption{Construction of the family $\mathcal{F}_\beta$, illustrated in the case of $\beta=2$.
    The ``snapshots'' of the curves on the two epochs $[x_j, x_{j+1}]$ and $[x_{j+1}, x_{j+2}]$. 
    For any combination of red or blue curves, the change at any endpoint is {\em smooth} -   both red and blue have $0$ derivative at any $x_j$.}
    \label{fig:instance}
\end{figure}

We describe this family at a high level using Figure~\ref{fig:instance}.
Fix a suitable $\delta >0$, which we will later choose to be $\tilde \Theta(T^{1/(2\beta+1)})$, we partition $[0,1]$ into {\em epochs} $[x_{j}, x_{j+1}]$ with $x_j = j \cdot 6\delta$ for $j=0,1,\dots, m-1$ where $m=1/(6\delta)$.
For each reward \func\ in this family, its restriction on each epoch is either  constant, on the order of $\delta^{\beta}$, or a {\em bowl shaped} curve; see the red and blue curve \resp\ in Figure \ref{fig:instance}.
The family $\cal F_\beta$ consists of all such $2^m$ choices.
A bowl curve $b$ on $[x_j, x_{j+1}]$ is $(\beta-1)$-times differentiable, with vanishing derivatives at $x_j$ and $x_{j+1}$, which enables a {\bf smooth} concatenation with a constant curve or another bowl curve. 

More concretely, in Figure \ref{fig:instance}, over the two epochs $[x_j, x_{j+1}]$ and $[x_{j+1}, x_{j+2}]$, an instance can correspond to each of the following $2^2=4$ combinations:
a constant curve (all red); 
the curve constituted by two bowl-shaped curves, one on each epoch (all blue);  
and the two combinations of constant and bowl-shaped curves (first red then blue, and first blue then red).

We now formalize these ideas.
As Figure \ref{fig:instance} illustrates, a bowl curve is obtained by connecting two {\em rotated} copies of $g_\delta$ with a constant \func, which we formally define as follows. Recall  from Proposition \ref{prop:pyramid} 
that $g_\delta(x) = C_\beta(\delta)\cdot \delta^\beta$. 

\begin{definition}[Construction of a Bowl] \label{def:bowl}
Fix any integer $\beta\geq 1$ and define
$\delta=\delta(\beta,T)=\lb(2^{2(\beta+1)} C_\beta^2 T\rb)^{-1/(2\beta+1)}.$
For $j=0,\dots,m-1$ and $x\in [0,1]$, define 
\[\mu_{j+1,\mathrm{r}}(x) = \frac 12 C_\beta(2\delta) \cdot (2\delta)^{\beta}\cdot \mathbbm{1}_{[x_j,x_{j+1})}(x),\] and   
\[\mu_{j+1,\mathrm{b}}(x) = \begin{cases}
-g_{2\delta}(x-x_j) + \frac 12 C_\beta (2\delta)^\beta, &\text{if } x\in [x_j, x_j +2\delta),\\
-\frac 12 C_\beta (2\delta)^\beta, &\text{if } x\in [x_j+2\delta, x_j +4\delta),\\
g_{2\delta}(x-x_j-4\delta) - \frac 12 C_\beta (2\delta)^\beta, &\text{if } x\in [x_j+4\delta, x_{j+1}),\\
0, &\text{else.}
\end{cases}
\] 
\end{definition}

Each \func\ in $\cal F_\beta$ corresponds to a binary vector that encodes its color in each epoch.
\begin{definition}[The Family $\mathcal{F}_\beta$]
For any $v=(v_1,\dots, v_m)=\{\mathrm{r}, \mathrm{b}\}^m$, let $\mu_v(x) = \sum_{j=1}^{m} \mu_{j, v_j}(x)$.
We define 
$\mathcal F_{\beta} = \{\mu_v: v\in\{\mathrm{r},\mathrm{b}\}^m\}$.
\end{definition}

It is \strfwd\ to verify using Proposition \ref{prop:pyramid} that for every configuration $v\in \{\pm 1\}^m$, the \func\ $\mu_v$ is $\beta$-H\"older. 
In fact, by (i), if $x$ is a multiple of $2\delta$, then for any $1\leq i\leq \beta-1$, the left and right order-$i$ derivatives at $x$ are both $0$, which ensures a smooth transition.

The construction above illustrates the role of $\beta$. Note that the total variation of the bowl-shaped curve is only $O(\delta^\beta)$.
Therefore, the more smoothness (i.e., $\beta$) we ask for, the less drastically the bowl curve can vary.
This suggests that the minimax regret should decrease in $\beta$. 
We formalize this intuition in the next subsection.

\subsection{The Main Lower Bound} \label{sec:lower}
Now we are ready to state the lower bound for the one-armed case. 
Recall that $m=1/(6\delta)$.
With some foresight, we choose $\delta=\Theta(T^{-1/(2\beta+1)})$.
We slightly abuse the notation by writing ${\rm Reg}(A,v)$ the regret of a policy $A$ when the mean reward of the changing arm is given by $r_1(t)= \mu_v (t/T)$.

\begin{theorem}[Main Lower bound]
\label{thm:lb}
For any integer $\beta\geq 1$ and policy $A$, there exists $v\in \{{\rm r}, {\rm b}\}^m$ such that
\[\mathrm{Reg}(A, v) > \frac 1{24}\cdot 2^{-\beta} \lb(C_\beta\rb)^{-\frac {2\beta}{2\beta+1}}\cdot T^{\frac{\beta+1}{2\beta+1}}.\]
\end{theorem}

We outline the analysis and defer the formal proof to Appendix \ref{apdx:lb}.
First, observe that due to the smooth concatenation, the learner is unable to predict whether in the upcoming epoch the curve is bowl-shaped or constant.
Thus, we can view the problem as $m$ separate instances, each having $6\delta T$ rounds, and hence it suffices to lower bound the regret on each epoch.

To do this, fix an epoch $[x_j, x_{j+1}]$. 
As a key step, observe that for any point in
$[x_j+2\delta, x_j + 4\delta]$, the optimal arm is arm $0$ if $r_1$ follows the blue curve, and arm $1$ \ow.
We show that any policy has  at least $1/2$ probability of choosing the wrong arm in each of these rounds.

We now make this statement formal. Consider two \ins s that are identical up to the $(j-1)^{\mathrm{st}}$ epoch. 
For any prefix $u\in \{\mathrm{r},\mathrm{b}\}^{j-1}$ and color $\chi\in \{\mathrm{r}, \mathrm{b}\}$, we will consider the \ins s $\mu_{u\oplus \mathrm{r}}$ and $\mu_{u\oplus \mathrm{b}}$ where $\oplus$ denotes the vector concatenation.
Recall that $t_j = x_j T$.

\begin{lemma}[Likely to Select a  Wrong Arm]
\label{lem:non-optimal}
For any round $t$ in the $j$-th epoch $[t_{j-1}, t_j]$, prefix $u\in \{\mathrm{r},\mathrm{b}\}^{j-1}$ and policy $A=(A_s)_{s\in [T]}$, it holds that \[\ho{P}_{u \oplus {\rm b}}[A_t = 1] + \ho{P}_{u \oplus {\rm r}}[A_t = 0] \ge \frac 12.\]
\end{lemma}

To show this, recall that the red and blue curves differ by $O(\delta^\beta)$, and that an epoch comprises only $\delta T$ rounds. 
Thus, any policy ``collects'' $ O(\delta^\beta)^2\cdot \delta T = O(1)$ Kullback–Leibler (KL) divergence within an epoch. 
Therefore, by Pinsker's \ineq, the red and blue curves cannot be distinguished, in the sense that the \prb\ of any event differs by at most $1/2$ under (the probability measure induced by) these two instances.
The desired \ineq\ follows by applying this to the event that $A_t=1$. 
We defer the details to Appendix \ref{apdx:non-optimal}.

Since there are $\delta T$ rounds in an epoch, the expected regret on an epoch is $\Omg(\delta^\beta) \cdot 6\delta T = \Omg(\delta^{\beta+1} T)$.
Summing over all $1/\delta$ epochs, the total regret is $\delta^{-1} \cdot \Omg(\delta^{\beta+1} T) = \Omg(\delta^\beta T) =\Omg(T^{(\beta+1)/(2\beta+1)})$.

\section{Upper Bounds for the One-armed Setting}\label{sec:ub}
In this section, we illustrate the key ideas behind our analysis in the {\em one-armed bandits} setting and generalize the results to the multi-armed setting in Section \ref{sec:k_arm}.
The mean rewards of arm $0$, the {\em static arm}, \sats\ $r_0(t)= 0$. 
Furthermore, there exists an unknown function $\mu_1\in \Sigma(\beta, L)$ such that the mean rewards of arm $1$, the {\em changing} arm, \sats\ $r_1(t)=\mu_1(t/T)$ for all $t\in [T]$.
The key is to determine the {\em sign} of the mean reward of the changing arm.
We will suppress the subscripts and write $\mu(t)=\mu_1(t)$ and $r(t)=r_1(t)$.

A natural policy is to estimate the rate of change in the reward rate and make decisions based on the predicted trend. 
Surprisingly, our policy, which achieves minimax optimal regret, does {\bf not} use any derivative information.
Instead, this policy initiates a ``restart'' at appropriate intervals, disregarding all previous observations.
We will provide a detailed description of this policy next.

\subsection{The Budgeted Exploration Policy}
As a rule of thumb in the non-stationary bandits, a reasonable policy should strike a balance between two trade-offs.  
The first is {\em exploration vs. exploitation}, i.e., balancing between information acquisition versus capitalizing on existing information.
The second is {\em remembering vs. forgetting}, i.e., discarding old information at a judicious rate to account for non-stationarity.

To manage the second trade-off, we propose a {\it Budgeted Exploration (BE)} policy that ``restarts'' periodically, as detailed in Algorithm \ref{alg:BE}.
The policy is specified by two parameters: the {\it exploration budget} $B \ge 1$ and the epoch size $\Delta \in (0,1)$. 
It partitions the normalized timescale $[0,1]$ into {\it epochs} $[x_i, x_{i+1}]$ where $x_i = i\Delta$ for each $i=0,\dots,\Delta^{-1}-1$. 
Equivalently, it partitions $\{1,\dots, T\}$ into epochs $[t_i, t_{i+1}]$ where $t_i = x_i T$. 
(For simplicity, we assume that $\Delta^{-1}$ is an integer. Apparently, this is not essential for the asymptotic results.)
In each epoch, the policy selects the changing arm from the start of the epoch until either (i) the epoch ends or (ii) the budget $B$ runs out, whereupon the policy selects the static arm in all remaining rounds in the epoch. 

Unlike the Rexp3 algorithm in \citealt{besbes2014stochastic}, which uses the EXP3 algorithm as a subroutine to handle exploration-exploitation, we employ a much simpler approach. 
We continue to select arm $1$ until its cumulative reward in this epoch is lower than $-B$.
This ensures that the average regret on each epoch is not too large.
This idea can also be generalized to the {\em multi}-armed setting by imposing a budget on the {\em difference} in the rewards of the empirically best arm. 

\begin{algorithm}[t!]
\begin{algorithmic}[1]
\For{$i=1,\dots,\Delta^{-1}$} 
\State $Z_{\rm total}\lar 0$  \Comment{Keep track of the total reward}
\State $A\lar 1$  \Comment{Which arm to play}
\For{$t=1,\dots,\Delta T$}
\State Select arm $A$ and observe reward $Z$
\State $Z_{\rm total}\rar Z_{\rm total} + Z$
\If{$Z_{\rm total} < -B$} $A\lar 0$ \Comment{Budget runs out}
\EndIf
\EndFor
\EndFor
\caption{Budgeted Exploration Policy $\mathrm{BE}(B,\Delta)$}
\label{alg:BE}
\end{algorithmic}
\end{algorithm}

We will soon see that for suitable $B$ and $\Delta$, the BE policy achieves minimax optimal regret for $\beta=1,2$.
We outline the ideas in the analysis in the next two subsections.

\subsection{Non-smooth Case: $\beta=1$}
\label{sec:pf_beta=1_k=1} Before delving deep into our focus, the $\beta=2$ case, we analyze the $\beta=1$ case as a warm-up. 

\begin{proposition}[Generic Upper Bound, $\beta=1,k=1$]\label{prop:ub_beta=1}
\Sps\ $6\Delta T\log T \le B^2$. Then \[\mathrm{Reg}\lb(\mathrm{BE}(\Delta, B),\Sigma_1(1,L)\rb) \leq \Delta^{-1} (1 + L\Delta^2 T + B).\]
\end{proposition}

By selecting $\Delta = L^{-2/3} T^{-1/3} \log^{1/3} T \text{ and } B=L^{-1/3} T^{1/3} \log^{2/3} T,$ we immediately obtain the following upper bound on the minimax regret for the $1$-H\"older family.
This result, combined with Theorem \ref{thm:lb}, characterizes the minimax regret for the $\beta=1$ case (up to a logarithmic factor in $T$). 

\begin{theorem}[Upper Bound, $\beta =1,k=1$]\label{thm:be_beta1}
For any $L>0$, there exists some $B=\tilde O(T^{1/3})$ and $\Delta=\tilde O(T^{-1/3})$ such that
\[\mathrm{Reg}\lb(\mathrm{BE}(B,\Delta), \Sigma_1 \lb({1,L}\rb)\rb) = O\lb(L^{1/3} T^{2/3}\log^{1/3} T\rb).\]
\end{theorem}

Note that any function in $\Sigma(1,L)$ has total variation $V\le L$ on $[0,1]$. 
Therefore, the above is implied by the $\tilde O(V^{1/3} T^{2/3})$ regret bound in \citealt{besbes2014stochastic}. 
However, we analyze this result to help lay the foundation for the analysis in the $\beta=2$ setting. 

Our analysis proceeds by analyzing the regret on three types of epoch.
Observe that in each epoch, the function $\mu$ either (i) is always positive, (ii) is always negative, or (iii) intersects the time axis, i.e., $\mu(x)=0$ for some $x$.
We refer to these three types of epochs as {\it positive}, {\it negative} and {\it crossing} epochs. We should mention that in a crossing epoch, $\mu$ may not necessarily ``cross'' the time axis but merely ``touches'' it; however, this distinction is not essential to our analysis.

We upper bound the expected regret incurred in each type of epoch as follows: 
\benum
\item First consider a positive epoch.
By concentration bounds, we can show that the budget is unlikely to be depleted, and hence with high probability the policy always chooses the correct arm. 
\item Now consider a negative epoch. 
\Sps\ the budget runs out in the $\tau$-th round after the epoch starts. 
By standard concentration bounds, we can show that with a high probability the policy eliminates the correct arm
and therefore the expected regret after time $\tau$ is $0$.
On the other hand, note that the total regret until time $\tau-1$ is at most $B$. 
By boundedness of the mean reward, we conclude that the regret until time $\tau$ is $B+O(1)$. 
\item Finally, consider a crossing epoch. 
Since $\mu$ is $L$-Lipschitz, we have $\mu(x) = O(L\Delta)$ on this epoch, where we recall that $\Delta \in [0,1]$ is the length of an epoch (on the {\em normalized} time scale). 
Thus, the total regret on this epoch is $ O(L\Delta)\cdot \Delta T = O(L\Delta^2 T)$. 
\eenum

Combining the analysis in the above three cases, we conclude that the regret on {\em every} epoch is $O(1+B+L\Delta^2 T)$. 
Proposition \ref{prop:ub_beta=1} follows immediately since there are $\Delta^{-1}$ epochs.
We formalize the above analysis in Lemma \ref{lem:no_crossing>0},
Lemma \ref{lem:no_crossing<0} and Lemma \ref{lem:regret_crossing_beta=1} respectively in Appendix \ref{apdx:beta=1}. 

\subsection{An $T^{3/5}$ Upper Bound for $\beta=2$}\label{sec:pf_beta=2_k=1}
Based on the above analysis framework, we next show an $\tilde O(T^{3/5})$ upper bound on the minimax regret for $2$-H\"older instances.
Again, we will show a generic bound first.

\begin{proposition}[Generic Upper Bound, $\beta=2$, $k=1$]
\label{prop:ub_beta=2}
\Sps\ $6\Delta T\log T \le B^2$, then \[\mathrm{Reg}\lb(\mathrm{BE}(\Delta, B),\Sigma_1(2,L)\rb) \leq 2\Delta^{-1} (1+B+L\Delta^3 T).\]
\end{proposition}

By choosing
$\Delta = L^{-2/5} T^{-1/5} \log^{1/5} T \text{ and } B=L^{-1/5}  T^{2/5} \log^{3/5} T$, we obtain the following main result of this section.

\begin{theorem}[Optimal Regret Bound, $\beta =2$]\label{thm:be_beta2}
For any $L>0$, there exist some $B=\tilde O(T^{2/5})$ and $\Delta=\tilde O(T^{-1/5})$ such that
\[\mathrm{Reg}\lb(\mathrm{BE}(B,\Delta), \Sigma_1 \lb({2,L}\rb)\rb) = O\lb(L^{1/5} T^{3/5} \log^{2/5} T\rb).\]
\end{theorem}

Remarkably, this upper bound provides the first {\em separation} between the smooth $(\beta\geq 2)$ and non-smooth $(\beta=1)$ regimes \citep{besbes2014stochastic}. 
Our analysis for $\beta=2$ follows the framework of the $\beta=1$ case. 
The analysis for positive and negative epochs remains valid, since they do {\em not} rely on the smoothness of $\mu$. 
We will refer to these two types of epochs as {\em non-crossing epochs}, as $\mu$ does not cross the $x$-axis.

The main difference lies in the regret analysis for crossing epochs. 
We will leverage the smoothness to argue that the regret on those epochs is low {\em on \avg}.
We distinguish between two types of crossing epoch, depending on whether they contain a zero of $\mu'$.

\begin{definition}[Stationary Points and Epochs]
A point $s\in [0,1]$ is {\it stationary} if $\mu'(s)=0$. 
An epoch $[x_j, x_{j+1}]$
is {\it stationary} if it contains a stationary point, and is {\it non-stationary} \ow.
\end{definition}

The regret analysis on a stationary epoch is relatively easier. 
By Lipschitzness of $\mu'$ and Newton-Leibniz theorem, on such an epoch we have $|\mu'| \le L\Delta$. 
Thus, for small $\Delta\in (0,1)$, the reward function would not vary much. 
In other words, the instance is almost stationary in this epoch and, therefore, easy to handle.

The key step is to bound the regret on {\em non-stationary} crossing epochs.
Imagine an adversary trying to fool the learner.
Knowing that the learner is using the BE policy, the adversary first chooses $\mu_1<0$ at the beginning of an epoch, so that when the exploration budget runs out, the static arm has a higher mean reward. 
The learner will then choose arm $0$ for exploitation.
After the learner stops exploring, the adversary bends $\mu_1$ upward to make it positive.
Consequently, the learner selects the wrong arm in the exploitation phase.

However, we next show that the adversary can not employ this tactic too often, due to the smoothness requirement.
We show that to generate a very high regret on a crossing epoch $i$, the nearest stationary point $x$ must be {\bf proportionally} far away. 
This enables us to {\em amortize} the regret on epoch $i$ to the epochs between $x$ and crossing epoch $i$.

\begin{lemma}[Key Lemma: Regret on Crossing Epochs]\label{lem:regret_crossing_beta=2}
Let $j$ be a crossing epoch, and $j+\ell$ be a stationary epoch ($\ell$ may be negative or $0$). 
Moreover, \sps\ every epoch between them is non-stationary, i.e., epoch $i$ is non-stationary whenever $(j+\ell-i)\cdot (j-i)<0$. 
Then, \[ \ho{E} \lb[\sum_{t=t_j}^{t_{j+1}} R_t\rb] \le 2L \cdot (|\ell|+1)\cdot \Delta^3 T.\]
\end{lemma}

We briefly explain the high-level idea.  
Denote by $\|f\|_\infty$ the supremum of the absolute value of a function $f$ on epoch $j$.
Since $\mu'$ is Lipschitz and vanishes at some point in epoch $j+\ell$, we deduce that $\|\mu'\|_\infty = O(L\cdot \ell\Delta)$.
Moreover, since $\mu$ has a crossing in epoch $j$, we have $\|\mu\|_\infty = O( \|\mu'\|_\infty \cdot \Delta)$. 
Therefore, the total regret is bounded by $\Delta T \cdot \|\mu\|_\infty=O( L\cdot \ell \Delta^3 T)$.

With this lemma, we next derive our main regret bound and defer the details to Appendix \ref{apdx:regret_crossing_beta=2}.

\noindent{\bf Proof Sketch of Proposition \ref{prop:ub_beta=2}.}
We show that the {\em \avg} regret per epoch is $O(L\Delta^3 T)$.
Let us index the stationary epochs as $1\le s_1< \dots< s_n\le \Delta^{-1}$. 
Crucially, we observe that for any $j$, there is at most one crossing epoch between $s_j$ and $s_{j+1}$. 
If it does exist, we denote by $i_{\mathrm{x}}$ the epoch that contains it.
Then, by Lemma \ref{lem:regret_crossing_beta=2}, the regret on this epoch \sats\
\[R[i_{\mathrm{x}}] \le 
2L\cdot (|i_{\mathrm{x}}-s_j|+1)\cdot \Delta^3 T \le 2L\cdot |s_{j+1}-s_j|\cdot \Delta^3 T.\]
Since the other epochs are non-crossing, the regret on each of them is $O(B)$. 
Thus, the total regret on epochs $s_j$ through $s_{j+1}$ can be bounded as
\[\sum_{s_j \le i < s_{j+1}}   R[i]\le (s_{j+1} - s_i) \cdot O(L\Delta^3 T + B).\]
Equivalently, the average regret on each epoch is $O(L\Delta^3 T + B)$.
\hfill$\square$

As a caveat, the above proof assumes that there is at least one stationary point. 
We finally handle the corner case where $\mu$ has no stationary point. 

\begin{proposition}[Corner Case: No Stationary Point]\label{lem:corner}
\Sps\ $\mu'(x)\neq 0$ for all $x\in [0,1]$. 
Then, for any $B>0$ and $\Delta\in (0,1)$, we have
\[\mathrm{Reg}\lb(\mathrm{BE}(B,\Delta),\mu\rb)\le L\Delta^2 T + (B+1)\Delta^{-1}.\]
\end{proposition}

Proposition \ref{lem:corner} combined with the above bound on the average regret together completes the proof of Proposition \ref{prop:ub_beta=2}.
We defer the detailed proof to Section \ref{apdx:corner_case}.

\section{Multi-Armed Setting}\label{sec:k_arm}
In this section, we consider the setting with $k\ge 2$ arms.
We modify the BE policy as follows: In each epoch, we alternate between the arms and eliminate an arm when its cumulative reward in this epoch is lower than that of another arm by (at least) $B$, the budget. 
If there is only one arm remaining, then we select it until the end of the epoch.
We formally state this policy in \Alg~\ref{alg:BE_2arm}. 

\begin{algorithm}[t!]
\begin{algorithmic}[1]
\State $t\lar 0$ \Comment{Initialize time}
\For{$i=1,\dots,\Delta^{-1}$} \Comment{Epochs}
\State For each $a\in [k]$, set $Z_a^{\rm total} \lar 0$ \Comment{{\bf Reset} cumulative rewards}
\State $A\lar [k]$ \Comment{{\bf Reset} alive arms}

\While{$|A|\ge 2$ and $t+|A| \leq  (i+1)\Delta T$}
\For{$a\in A$} \Comment{Round robin selection of the alive arms}
\State Select $a$ and observe reward $Z_a^t$
\State $Z_a^{\rm total}\lar Z_a^{\rm total} + Z_a^t$; \quad $t\lar t+1$ \Comment{Update cumulative reward}
\EndFor
\State $Z_{\rm max}^{\rm total} \lar \max_{a\in A} Z_a^{\rm total}$\Comment{The largest cumulative reward}

\For{$a\in A$}
\If{$Z_a^{\rm total} < Z_{\rm max}^{\rm total} - B$} \Comment{Arm $a$ has run out of budget} \State $A\lar A\bs \{a\}$ \Comment{Remove $a$}
\EndIf
\EndFor
\EndWhile

\While{$t< (i+1)\Delta T$} \Comment{Exploit the only alive arm until the epoch ends}
\State Select the only arm in $A$;\quad $t\lar t+1$
\EndWhile

\EndFor
\caption{The Budgeted Exploration Policy, Finite-Armed Case}
\label{alg:BE_2arm}
\end{algorithmic}
\end{algorithm}

The main result of this section is a regret bound that is (i) minimax nearly optimal in $T$ and (ii) sublinear in $k$.
As in the one-armed case, we start with a regret bound for \arb\ $B,\Delta$.

\begin{proposition}[Generic Upper Bound, $\beta=2$, $k\ge 2$]\label{prop:k-arm}
\Sps\ $B^2 \ge k^{-1} \Delta T \log T \log k$. 
Then for any finite $k$, Lipschitz constant $L>0$, we have 
\[{\rm Reg}\lb({\rm BE}(B,\Delta),\ \Sigma_k(2,L)\rb) \le  2\Delta^{-1} \lb(1+ k B +  k^2 L \Delta^3 T\rb).\]
\end{proposition}

It is important to note that, in contrast to the bound outlined in Proposition \ref{prop:ub_beta=2}, we have introduced an additional factor of $k$ before $B$ and $k^2$ before $L\Delta^3 T$. 
By choosing \[\Delta = L^{-2/5} k^{-3/5} T^{-1/5} \cdot\log^{1/5} T \cdot \log^{1/5} k \quad \text{and} \quad B = L^{-1/5} \sqrt{ k^{-1} \Delta T \log T \log k},\]
we immediately obtain the following.

\begin{theorem}[Upper Bound, $\beta=2$, $k\ge 1$]\label{thm:ub_k_arms} For any finite $k\ge 1$, Lipschitz constant $L>0$, there exists some $\Delta = \tilde O(k^{-3/5} T^{-1/5})$ and $B=\tilde O(k^{-1/2}T^{2/5})$ such that
\[{\rm Reg}\lb({\rm BE}(B,\Delta), \Sigma_k(2,L)\rb) = O\lb( k^{4/5} T^{3/5}\cdot \log^{2/5} k\cdot \log^{2/5} T\rb).\]
\end{theorem}
 
Our analysis employs a {\em potential function} argument, a common technique in competitive analysis.
Imagine an adversary who constructs the reward function to fool the learner.
The adversary starts with a certain amount of {\em potential}.
We argue that every time the adversary forges an epoch with high regret, the potential decreases {\em proportionally} to the regret on this epoch. 
Consequently, the adversary has less power in the future.
Therefore, we can bound the regret in terms of the {\em initial} potential.

The critical step in this argument is to find an appropriate definition of potential. 
The pivotal insight is derived from the key lemma in the one-armed case (Lemma \ref{lem:regret_crossing_beta=2}):
The only way for the adversary to generate a $\omg(T^{2/5})$ regret in an epoch is to create a crossing with a high derivative. 
(We write $f(x) = \omg(g(x))$ if $g(x) = o(f(x))$.)

To generalize this idea to the multi-armed setting, we re-define a {\em crossing} as a point where the {\em best} arm changes. 
Furthermore, we define a {\em fast crossing} as a crossing where the derivatives of the two arms involved differ by $\Omg(\Delta)$.

We classify the epochs into three types: (i) non-crossing epochs, in which the best arm never changes;  
(ii) slow-crossing epochs, which contain a crossing but no fast crossings; and (iii) fast-crossing epochs, which contain at least one fast crossing.

The regret analysis for the first two types of epoch is similar to the one-armed setting. 
To analyze the regret in the third type, we define the {\em energy} at a fast crossing as the difference in $\mu'$ between the two arms involved.
Repeating the argument in Lemma \ref{lem:regret_crossing_beta=2} for all {\em ordered} pairs of arms, we show that, for any instance, the total energy is at most $k(k-1)L$. 
By choosing this to be the initial potential, we can bound the total regret on all type-(iii) epochs as $O(k^2 L\Delta^2 T)$. 
We defer the details to Appendix \ref{apdx:k_arm}.

\section{Experiments}\label{sec:expmt}
In this section, we study the practical impact of smoothness in non-stationary bandits via experiments on synthetic data (in Section \ref{sec:synthetic}) and real data (in Section \ref{sec:yahoo}) \resp. 
We evaluate the performance of our BE policy under the theoretical order optimal \pmt s for non-smooth setting ($\beta = 1$) and smooth setting ($\beta =2$) \resp, given in Theorem \ref{thm:be_beta1} and Theorem \ref{thm:be_beta2}, and denoted by BE-NS and BE-S.
As a natural benchmark, we also implemented the Rexp3 policy in \citealt{besbes2014stochastic}. 
We find that the BE-S policy, the only policy in the experiments that leverages smoothness, consistently outperforms the other two benchmarks.

\subsection{Additional Benchmark: Rexp3 policy}
Apart from the BE-NS policy, we consider another benchmark, Rexp3, proposed by \cite{besbes2014stochastic}. 
To explain this policy, we first provide some background.
In the {\em adversarial bandits} problem, the reward \func s are chosen by an adversary sequentially. 
A natural approach is the {\em multiplicative weights algorithm} (MWA).
In simple terms, it maintains a weight for each arm and randomly draws an arm in each round with a probability proportional to the weight.

However, to update the weights, we need to assume {\em full information} feedback, i.e., we observe the rewards of all arms regardless of which arm is selected. 
To operate under {\em bandit feedback}, where we only observe the reward of the selected arm, the EXP3 algorithm was introduced.
This policy ``simulates'' MWA using unbiased reward estimators and achieves $\tilde O(\sqrt T)$ expected regret against the best arm in hindsight. 

To incorporate the forgetting-vs-remembering trade-off, \cite{besbes2014stochastic} proposed the ``restart'' version of the EXP3 \alg,  dubbed {\em Rexp3}.
This \alg\ partitions the horizon into epochs (or {\em batches}, as in their work) of length $\Delta_T$, and
ignores all previous observations at the beginning of each epoch.
Moreover, for robustness in the performance, they integrated this policy with the $\eps$-greedy \alg. 
The policy randomly uniformly explores an arm in each round with a fixed \prb\ $\gamma$. 
This policy achieves $\tilde O(V^{1/3}T^{2/3})$ regret where $V$ is the total variation of the mean rewards; see their Theorem 1.

\subsection{Synthetic Experiments}\label{sec:synthetic}

\begin{figure}[t!]
    \centering
    \includegraphics[width=0.9\linewidth]{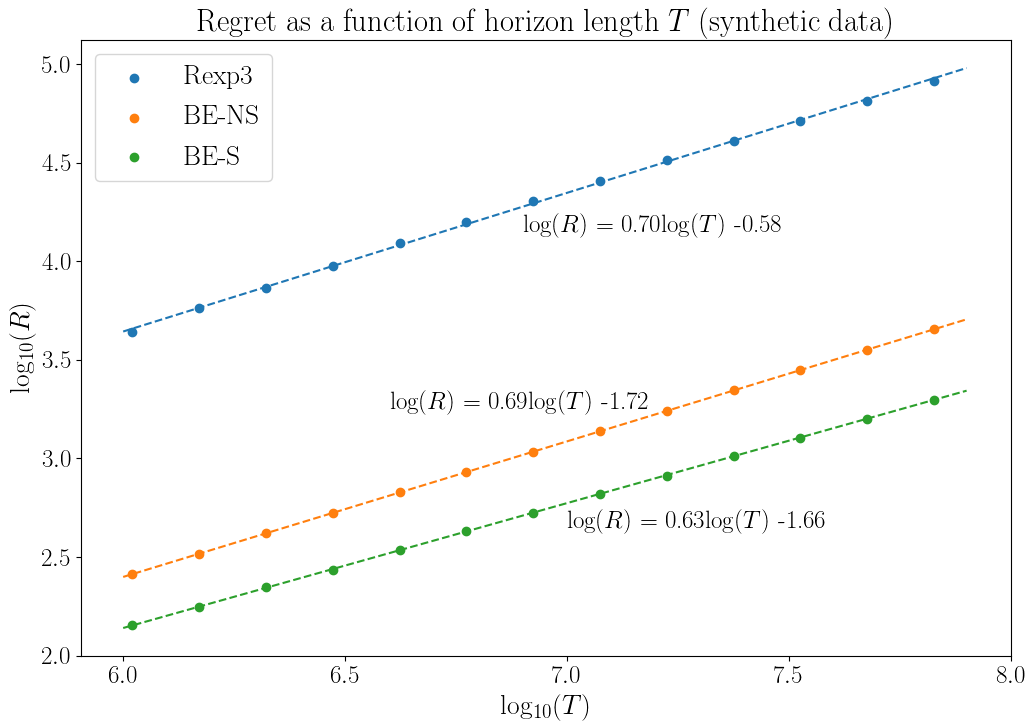}
    \caption{{\bf Log-log regret plot on synthetic data.} To visualize how the regret $R$ of the policies (BE-NS, BE-S and Rexp3) scale in the length $T$ of the time horizon, we present a log-log  (base 10) plot. 
    Each data point represents the regret of a policy, averaged across $100$ randomly generated sinusoidal instances.
    We applied linear regression to the data points corresponding to each policy and obtained three linear curves, whose expressions are provided in the figure. 
    The slopes of these curves align closely with the theoretical values. 
    In particular, the regret of BE-S grows considerably more slowly than the benchmarks.}
    \label{fig:synthetic}
\end{figure}

We first implemented our algorithm with simulations on synthetic data in the one-armed setting. 
We consider random {\it trigonometric} reward functions whose amplitudes, frequencies, and phase shifts are randomly drawn. 
Specifically, in each instance, we have \[r_0(t)=A 
\quad\text{and}\quad r_1(t)=-A \cdot \textrm{sin}\lb(2\pi \nu \frac tT+\phi\rb)+A,\] where $\nu\sim\mathcal{U}_{[2.5,5]}$,  $A\sim\mathcal{N} (0.25\nu^{-2},0.001)$ and $\phi\sim\mathcal{U}_{[0,2\pi]}$.

The meticulous reader may have noticed that $A$ depends on $\nu$.
This choice is actually quite natural. 
In fact, consider $\mu(x) = -A \sin(2\pi \nu x +\phi) + A$.
Note that $\mu''$ has a $\nu^2 A$ term, so by choosing $A$ to scale as $\nu^{-2}$, the \func\ $|\mu''|$ is bounded by an absolute constant, and therefore $\mu$ is 2-H\"older.

We chose the \pmt s for our BE  policy to rate optimal as given in Theorem \ref{thm:be_beta1} and Theorem \ref{thm:be_beta2}.
We also implemented the Rexp3 policy with a total variation budget of $V=0.05$ with order-optimal \pmt s, as given in \citealt{besbes2014stochastic}. Specifically, we chose the epoch size $\Delta_T = \lceil k\log^{1/3}k\cdot (T/V)^{2/3}\rceil$ and the exploration probability $\gamma = \min\left(1, \sqrt{\frac{k\log k}{(e-1)\Delta_T}}\right)$. 

The regret of the policies is visualized using a log-log plot with base 10 in Figure~\ref{fig:synthetic}. 
The time horizon ranges from $T=10^6$ to $T=10^8$.
Theoretically, the {\bf slope} of a log-log curve should be equal to the exponent of the cumulative regret. 
In fact, if the cumulative regret is $c T^d$, then the log-regret is $\log c + d \log T$.
Our simulation shows that under smooth non-stationarity, the $T^{3/5}$-regret policy BE-S outperforms the other two $T^{2/3}$-regret policies (BE-NS and Rexp3). 
Moreover, the log-log curves have slopes $0.69$ and $0.63$, respectively, which are close to their theoretical values.

\subsection{Experiments on Real data}\label{sec:yahoo}
We further investigate the effectiveness of our policy using the R6B subset of {\em Yahoo! Front Page Today Module User Click Log Data Set}  \citep{yahoo_webscope}. 
This data set contains user click log data for news articles (or {\em article}, for simplicity) displayed in the {\em Today Module's} Feature Tab on {\em Yahoo! Front Page}.
Each entry in the data set consists of the timestamp of a visit, the ID of the displayed news article, and a binary indicator of whether a click occurred.
We visualize the click-through rates (CTR) using rolling window \avg\ with window size $1$ hour; see Figure \ref{fig:CTR}.

This data set is well suited for our problem:
Each news article can be viewed as an arm and each click-through can be viewed as a reward.
Specifically, we implemented the policies for $k=20$ advertisements on 2 October, 2011. 

\begin{figure}[htbp]
\centering
    \includegraphics[height=0.35\linewidth]{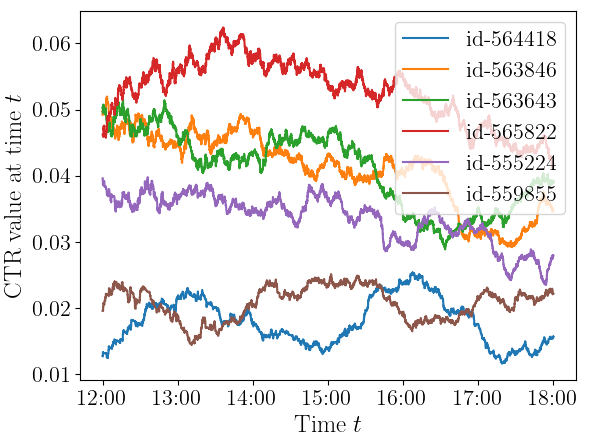}
    \includegraphics[height=0.35\linewidth]{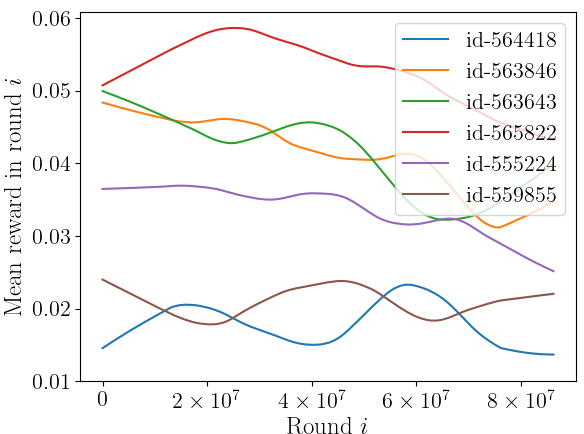}
    \caption{{\bf Modeling non-stationarity in the CTR using {\em Yahoo!} data.} We first employ a rolling window average method on the {\em Yahoo!} user-click data to obtain a {\em non-smooth} function that represents the variations of CTR in time, as illustrated in the left subfigure. 
    In the second part of our experiment, we smooth these functions using local regression, resulting in a mean reward sequence of length $8.64\times 10^7$, where each round corresponds to a second; see the right subfigure.
    }
    \label{fig:CTR}
\end{figure}

\noindent\textbf{Parameters.} 
We choose rate-optimal parameters for every policy in our experiments. 
In particular, we implemented the Rexp3 policy with a total variation budget of $V_T=0.05$, epoch size (or {\em batch size}, as in \citealt{besbes2014stochastic}) $\Delta_T = \lceil k\log^{1/3}k\cdot (T/V_T)^{2/3}\rceil$, and exploration probability $\gamma = \min\left(1, \sqrt{\frac{k\log k}{(e-1)\Delta_T}}\right)$ as proposed in \citealt{besbes2014stochastic}. 
The two BE policies are implemented with parameters specified in Theorem~\ref{thm:be_beta1} and Theorem \ref{thm:be_beta2}:
\begin{itemize}
    \item BE-NS: budget $B_T=O(T^{1/3}\log^{2/3}T)$, epoch size $\Delta_T=O(T^{2/3}\log^{1/3}T)$;
    \item BE-S: budget $B_T=O(T^{3/5}\log^{4/5}T)$, epoch size $\Delta_T=O(T^{4/5}\log^{1/5}T)$.
\end{itemize}

\begin{figure}[htbp]
    \centering
    \includegraphics[width=150mm]{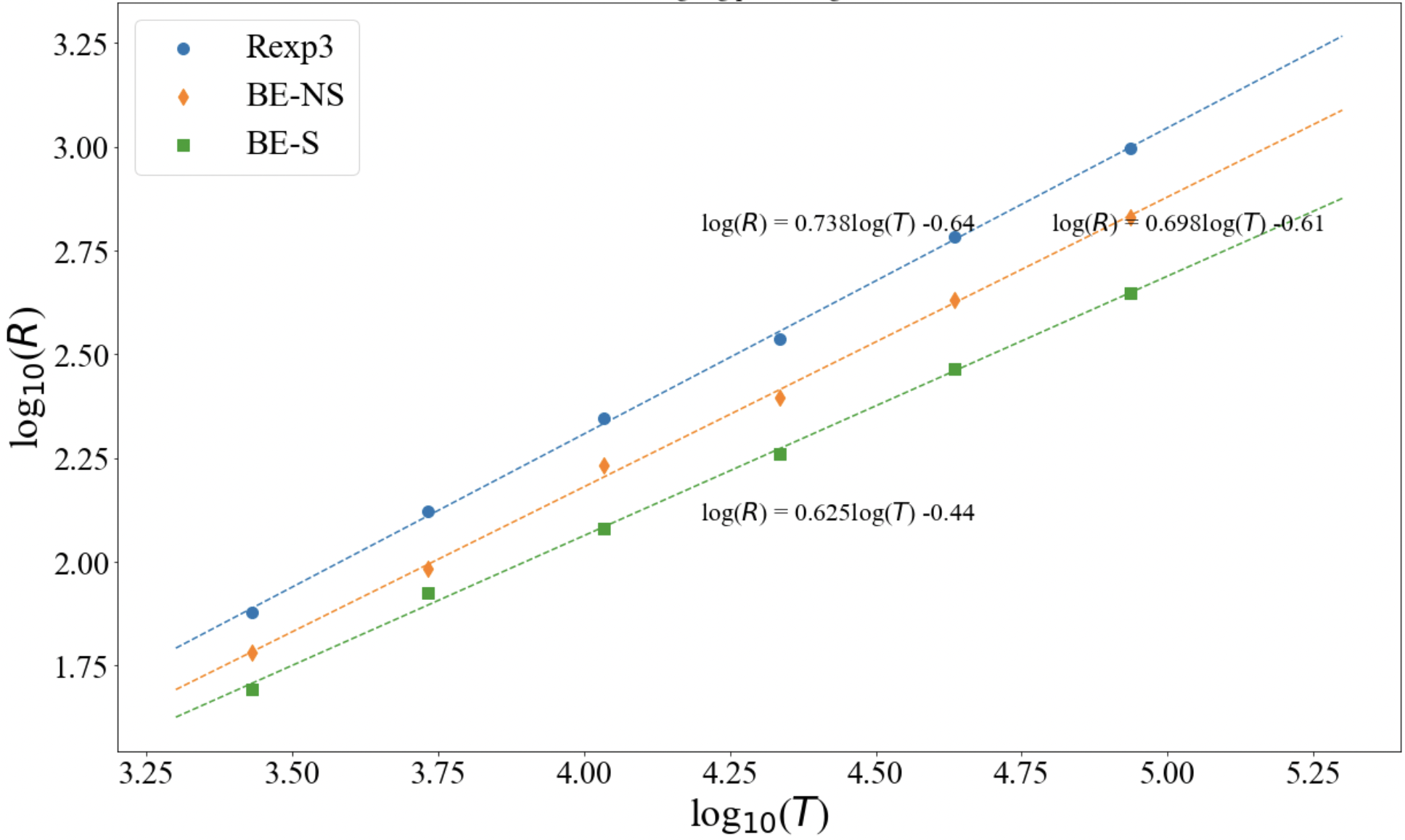}
    \label{fig:xp_counterfactual}
    \caption{Visualization of the experimental results in the counterfactual setting.}
\end{figure}   

\begin{figure}[htbp]
    \centering \includegraphics[width=160mm]{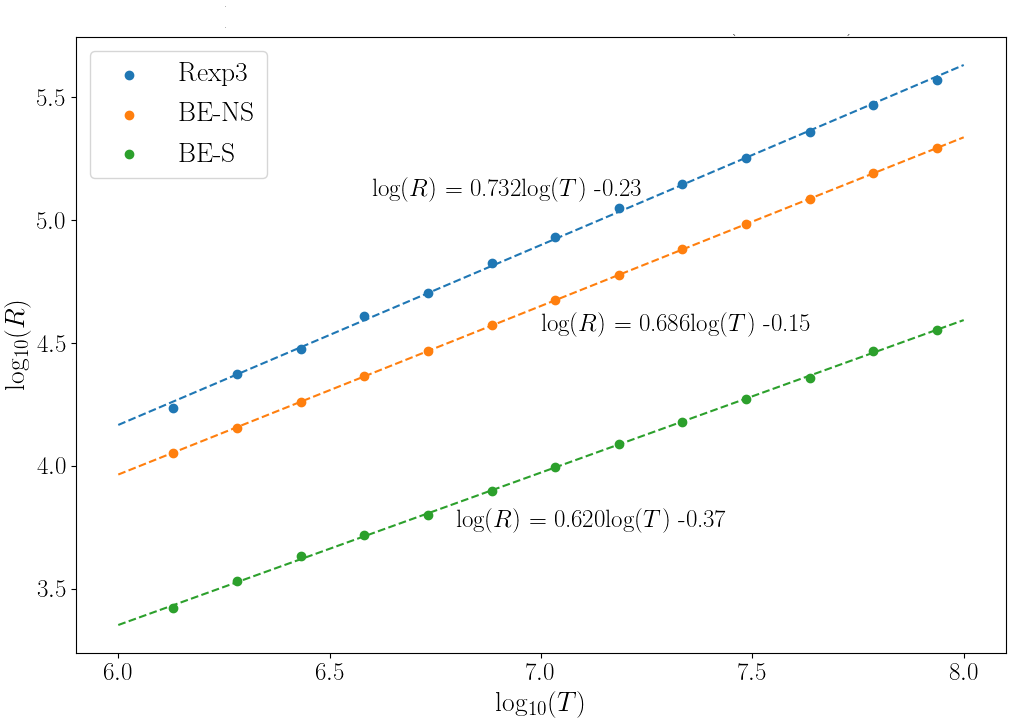}  
    \caption{Visualization of the experimental results for the smoothed reward functions.}
    \label{fig:xp_smoothed}
\end{figure}

\noindent{\bf Counterfactual Experiment.}
We first implemented a counterfactual experiment. 
In this setting, each round corresponds to a second.
In each round $t$, for each ad, we uniformly randomly draw an interaction contained in the rolling window of 1 hour, centered at time $t$ (i.e., 30 minutes before and after $t$), and use its click-through as the reward.
We implemented the three policies with time horizons $T = 24 \cdot 2^{-i}$ hours, where $i=0,1,\dots, 5$. 
In this way, the shortest and longest time horizons involve
$2700$ and $86400$ seconds \resp. 
We visualize their performance in Figure~\ref{fig:xp_counterfactual}.

It should be noted that in this counterfactual experiment, we do not explicitly assume smoothness in the non-stationarity. 
However, the BE-S policy continues to outperform both the Rexp3 and BE-NS policy.
This suggests that despite the zigzag pattern in the average CTR, it is reasonable to assume that the underlying non-stationarity exhibits a smooth behavior.
Furthermore, the log-log curves of the BE-S and BE-NS policies have slopes $0.698$ and $0.625$, which are close to their theoretical values. \\

\noindent{\bf Fitted instance.} We further validate the effectiveness of our policy by fitting a smooth reward function for each ad using local linear regression.
Specifically, we applied the {\em Locally Weighted Scatterplot Smoothing} (LOWESS) method \citep{cleveland1981lowess}. 
For each article, we compute its CTR as a function of time using a rolling window of $1$ hour. 
Specifically, the CTR at a particular time point is the number of user clicks divided by the number of users it has shown to, during an hour before this time point. 
The resulting CTR curves are shown in Figure \ref{fig:CTR}. \\

\noindent\textbf{Smoothed Reward Functions.} We measure the regret of the three policies on time horizons whose lengths range from $T=10^6$ to $T=10^8$; see Figure~\ref{fig:xp_smoothed}.
Consistent with the counterfactual experiment, the BE-S policy, which leverages the inherent smoothness, has lower regret compared to the Rexp3 and BE-NS policies, which do not exploit the smoothness of the non-stationary environment. 
Furthermore, the log-log curves of the BE-S and BE-NS policies have slopes $0.686$ and $0.620$, which are close to their theoretical values.
These findings validate the theoretical effectiveness of our policy.

\section{Conclusions and Future Directions}\label{sec:conclusion}
In this paper, we presented {\it smoothly-varying} non-stationary bandits and demonstrated the first separation between the smooth and non-smooth settings. 
We present a derivative-free policy that achieves $\tilde O(T^{3/5})$  regret in the smooth setting, which is asymptotically lower than the $\tilde O(T^{2/3})$ minimax regret for the non-smooth regime.
Moreover, we show that this upper bound is nearly optimal by establishing an  $\Omg(T^{(\beta+1)/(2\beta+1)})$ lower bound for the minimax regret on the $\beta$-H\"older family for every integer $\beta\ge 1$.
Finally, we conjecture that the lower bounds can be matched for every integer $\beta \ge 3$ but this remains open. If this is true, it means that as smoothness increases, we can obtain regret that approaches the optimal $\tilde O(\sqrt T)$ regret of {\bf stationary} bandits, since $\tilde O(T^{(\beta+1)/(2\beta+1)}) = T^{1/2 + O(1/\beta)}$.

\paragraph{Acknowledgements}
Peter Frazier was supported by AFOSR FA9550-19-1-0283 and FA9550-20-1-0351.
Nathan Kallus is partially supported by the National Science Foundation under Grant No. 1846210.

\bibliography{bandits}
\bibliographystyle{plainnat}

\appendix
\onecolumn
\newpage
\section{Proof of Proposition \ref{prop:pyramid}}
\label{apdx:lb}
In this section we provide details for constructing the family in the lower bound proof.

\subsection{Preliminaries: the Flock Transformation and the Pyramid}
We need the notion of flocks to construct bowl-shaped curves.
Pictorially, the flock transformation of a given function $h$ (called the {\it base} \func) is given by a \xulie\ of copies of $h$ side by side, each weighted by a constant.
For example, the topmost subfigure in Figure \ref{fig:pyramid_flock} is a flock transformation of the  pyramid-shaped base \func. 

\begin{definition}[Flock Transformation]
For any {\it base \func} $h(x):[0,w]\rar \real$ and {\it weight} vector $v\in \real^\ell$, the {\it $v$-flock} is a \func\ $F_v[h]: [0, \ell w] \rar \real$ given by
\[F_v[h](x) = \sum_{i=1}^\ell v_i \cdot h\lb((i-1)w + x\rb).\]
\end{definition}

We now specify the family of functions $g_\eps$ for constructing a bowl-shaped curve. 
We will set the highest derivative, i.e., $g_\eps^{(\beta-1)}$, to be a flock transformation of the pyramid \func\ whose weight vector is chosen from among the following {\it neutralizing} vectors. 

\begin{definition}[Neutralizing Vectors]
Let $\nu^0 = 1$ and $\oplus$ be  vector concatenation. 
For each integer $k\geq 1$, recursively define the $k$-th {\it neutralizing vector}  as $\nu^k = \nu^{k-1} \oplus (-\nu^{k-1})\in \{\pm 1\}^{2^k}$.
A flock \corres\ to a neutralizing vector is called a {\it neutral flock}.
\end{definition}

As the name suggests, these vectors have the property that the sum of their entries on any dyadic interval is $0$. 
Formally, $\sum_{j2^i+1}^{(j+1)2^i} \nu_i^k=0$ for any $\nu=\nu^k$ and integers $i\ge 1$, $j\ge 0$.  

The following result on the transformation of the flock and neutral vectors will be useful for showing Proposition \ref{prop:pyramid}.
The proof of the above is rather \strfwd\ and we leave it to the readers.

\begin{lemma}[Algebra of the Flock Transform]
\label{lem:flock_algebra}
Denote by $\circ$ the composition of mappings.
For any integers $i,j,k\ge 0$, it holds that\\
\rm (i) (Distributive Law)  $\lb(F_{\nu^i}\circ F_{\nu^j}\rb) \circ F_{\nu^k} =  F_{\nu^i} \circ \lb(F_{\nu^j}\circ F_{\nu^k}\rb)$,\\
\rm (ii) (Additive Law)  $F_{\nu^i} \circ F_{\nu^j} = F_{\nu^{i+j}}=F_{\nu^j} \circ F_{\nu^i}$.
\end{lemma}

In our construction, we set the highest derivative of $g_\eps$ to be a neutral flock of pyramid \func s defined as follows. 

\begin{definition}[Pyramid Function]
For any $w>0$, define the {\it $w$-pyramid} \func\ as 
\[\Delta_w(x)= x\cdot\mathbbm{1}_{\lb[0,\frac w2\rb]}(x)+ \lb(\frac w2- x\rb) \cdot \mathbbm{1}_{\lb[\frac w2,w\rb]}(x).\]
\end{definition}

\subsection{Construction via Anti-Derivatives}

\begin{figure}
\centering \includegraphics[width=0.5\linewidth]{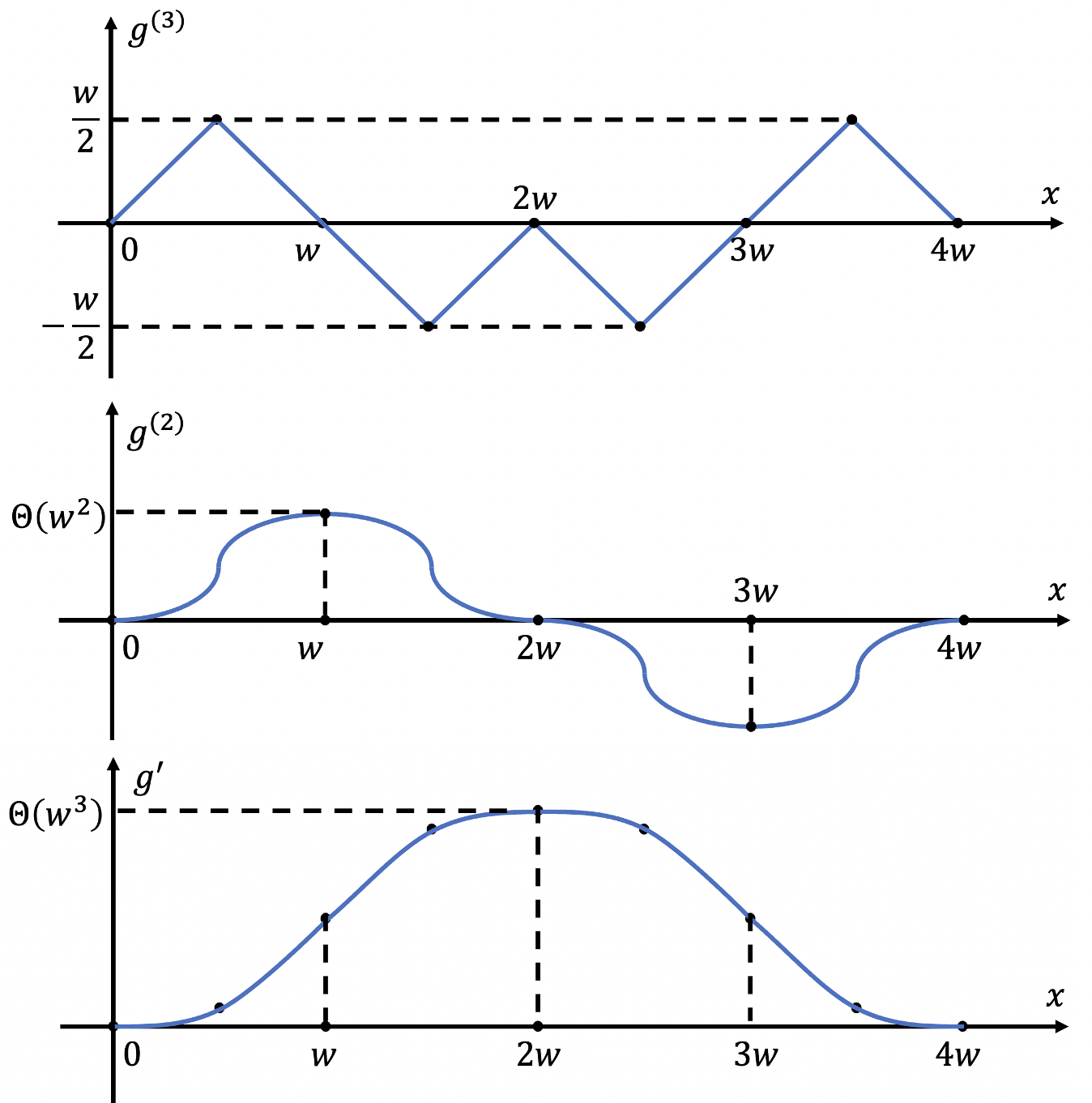}
\caption{Illustration of $g_\eps$ for $\beta = 4$: $g^{(3)}$ is a pyramid flock. The \func\ $g^{(2)}(x)$ is defined as the integration of $g^{(3)}$ from $0$ to $x$. 
Similarly, $g^{(1)}(x)$ is the integration of $g'$ from $0$ to $x$.}
\label{fig:pyramid_flock}
\end{figure}


The above construction can be formalized using anti-derivatives, which we define below.

\begin{definition}[Anti-derivative]
Consider any Lebesgue integrable \func\ $f:\real_{\ge 0} \rar \real$. 
We define $\Phi^0[f] = f$ and for any integer $\ell \geq 1$, the {\it level-$\ell$ anti-derivative}   $\Phi^\ell[f]: \real \rar \real$ is given by 
\[\Phi^\ell[f] (x) = \int_0^{x} \int_0^{x_\ell}\dots \int_0^{x_2} f(x_1)\ dx_1 \dots dx_\ell.\]
\end{definition}

Now consider $g_\eps = \Phi^{\beta-1} \lb[F_{\nu^{\beta -2}}\lb[\Delta_w\rb]\rb]$ 
with $w= w(\eps) = 2^{-(\beta-1)} \eps.$
Note that the $(\beta-1)$-st derivative is $F_{\nu^{\beta-1}}[\Delta_w]$, which is $1$-Lipschitz, so property (iv) holds trivially.
Moreover, observe that $w 2^{\beta-1}=\eps$, and hence  $g_\eps$ is supported on $[0,\eps]$ as desired.

We next formally verify that $\{g_\eps\}$ has the other properties claimed in Proposition \ref{prop:pyramid}.
The next result, which connects the notions of anti-derivatives, flock transformation and neutralizing vectors, says that the anti-derivative of a neutral flock is still a neutral flock.

\begin{proposition}[Anti-derivative Preserves Neutrality]
\label{prop:neutrality}
Let  $h_0:[0,w]\rar \real$ be any base \func\ and let $j\ge 0$ be any integer.
Then, for any $0\le \ell\le k$, there exists some base \func\ $h_\ell:[0,2^\ell w]\rar \real$ such that 
\[\Phi^\ell \lb[F_{\nu^{\ell+j}}\lb[h_0\rb]\rb] = F_{\nu^j}[h_\ell].\]
More precisely, we have $h_\ell = \Phi^\ell[F_{\nu^\ell}[h_0]]$.
\end{proposition}
To see the intuition, consider again $\beta =4$ (readers  can  again refer  to Figure \ref{fig:pyramid_flock}) and consider pyramid base \func\ $h_0 = \Delta_w$.
In this case, the highest derivative is given by a $\nu^2$-flock of pyramids,  i.e., 
\begin{align}
\label{eqn:011723a}
g^{(3)} =F_{\nu^2}[h_0].
\end{align}
Now consider $g^{(2)}$.
On one hand, as illustrated in the middle subfigure in Figure \ref{fig:pyramid_flock}, $g^{(2)}$ is a $\nu^1$-flock under the base \func\ $h_1 = \Phi^1[F_{\nu^1}[\Delta_w]]$, i.e.,
\begin{align}\label{eqn:011723b}
g^{(2)} = F_{\nu^1}[h_1].
\end{align}
\OTOH, by   \eqref{eqn:011723a} we also have \[g^{(2)} =\Phi^1 [g^{(3)}] = \Phi^1 [F_{\nu^2}[h_0]].\]  
Combining with \eqref{eqn:011723b}, we have
\[\Phi^1 [F_{\nu^2}[h_0]] = F_{\nu^1}[h_1],\]
as claimed for $\ell=j=1$.

\noindent{\bf Proof of Proposition \ref{prop:neutrality}.}
Consider induction on $j$.
Wlog assume $w=1$. 
The base case, $j=0$, is trivially true, since $F_{\nu^0}$ is  the identity mapping.
Now consider $j\ge 1$.
As the \ih, \sps\ the claim holds, i.e., 
\begin{align*}  \Phi^\ell \circ F_{\nu^{\ell+i}} &= F_{\nu^i} \circ \Phi^\ell \circ  F_{\nu^\ell},
\end{align*}
for $i=0,\dots, j-1$.
Denote by IH$(i)$ the above identity (which serves as the \ih).
Then, for any base \func\ $h$, it holds  that 
\begin{align*}
\Phi^\ell \circ F_{\nu^{\ell+j}}[h] & = \Phi^\ell \circ F_{\nu^{\ell+j-1}} [F_{\nu^1}[h]]\\
&= F_{\nu^{j-1}} \circ  \Phi^\ell F_{\nu^\ell}[  F_{\nu^1}[h]] &\text{by IH(} j-1)\\
&= F_{\nu^{j-1}} \circ  \Phi^\ell F_{\nu^{\ell+1}} [h] &\text{by Lemma \ref{lem:flock_algebra}}\\
& = F_{\nu^{j-1}}\circ F_{\nu^1} \circ \Phi^\ell \circ  F_{\nu^\ell} [h] &\text{by IH(} 1)\\
&= F_{\nu^j} \circ \Phi^\ell \circ  F_{\nu^\ell} [h], &\text{by Lemma \ref{lem:flock_algebra}}
\end{align*}
and the induction is completed.
\hfill$\square$

We use Proposition \ref{prop:neutrality} to verify properties (i) and (ii) in Proposition \ref{prop:pyramid}.
To verify that the derivatives do vanish at the endpoints, we need the following nice property of the neutralizing vectors.
\begin{lemma}[Symmetric Area Property]\label{lem:sap}
For any base \func\ $h:[0,w]\rar \real$ and integer $k\geq 1$, the $\nu^k$-flock \sats\ $\int_0^{2^k w} F_{\nu^k}[h](x)\ dx=0.$
\end{lemma}
\proof{Proof.}
Induction on $k$. 
For $k=1$ this is obviously true. 
As the \ih, \sps\ this is true for some $k\ge 2$. Recall that by definition it holds that $\nu^k = \nu^{k-1} \oplus (-\nu^{k-1})$, so 
\begin{align*}
& \int_0^{2^{k+1} w} f_{w,\nu^{k+1}}(x)\ dx \\
= & \int_0^{2^k w} f_{w,\nu^k}(x) \ dx  + \int_{2^k w}^{2^{k+1} w} f_{w,-\nu^k}(2^k w + x) \ dx\\
= & \int_0^{2^k w} f_{w,\nu^k}(x) \ dx  - \int_0^{2^k w} f_{w,\nu^k}(x) \ dx\\
=&\ 0.
\end{align*}
\hfill$\square$

\begin{proposition}[Vanishing Derivatives at the Endpoints] 
Let $h: [0,w]\rar \real$ be any base \func\ 
and $H=F_{\nu^\ell}[h]$.
Let $g=\Phi^\ell (H)$.
Then, the \func\ $g:[0,2^\ell w]\rar \real$ is $\ell$-times \conti ly differentiable with
$g^{(j)}(0)=g^{(j)}(2^\ell w)=0$ for any $j=1,\dots,\ell$.
\end{proposition}
\proof{Proof.} By Proposition \ref{prop:neutrality}, for any $j$,  there exists base \func\ $h_j:[0,2^\ell w]\rar \real$ such that 
$g^{(\ell-j)} = \Phi^j [H] = F_{\nu^{\ell-j}} [h_j].$
By Lemma \ref{lem:sap}, \ift\   
\[g^{(\ell-j-1)}(2^\ell w) = g^{(\ell-j-1)}(2^\ell w) - g^{(\ell-j-1)}(0)
= \int_0^{2^\ell w} g^{(\ell-j)} = 0.\eqno\]
\hfill$\square$

Proposition \ref{prop:pyramid} then follows immediately since we have verified properties (i)-(iv).

\section{Proof of Lemma \ref{lem:non-optimal}}\label{apdx:non-optimal}
First, we introduce some standard concepts and tools.
For simplicity, for any \ins\ $\mu:[0,1]\rar \real$, let $(Z_\mu^t)_{t\in [T]}$ be the reward vector under this \ins.

\begin{definition}[KL-Divergence]\label{def:KL}
Let $X, Y\in \{\pm 1\}^n$ be random vectors, specified by \prb\ mass \func s $f_X, f_Y: \{\pm 1\}^n\rar [0,1]$.
The {\it Kullback-Leibler divergence} (or KL divergence) is defined as
\[\mathrm{KL}(X,Y) = \sum_{v\in \{\pm 1\}^n} f_X(v) \log \frac{f_X(v)}{f_Y(v)}.\]
\end{definition}

We show that at any time, the KL-divergence of the two \rv s is on the order of the squared difference of their means. 
\begin{lemma}[Bounding the KL-Divergence]\label{lem:RC}
For $i=1,2$, \sps\ \rv\ $Z_i$ takes value on $\{\pm 1\}$ and has mean $r_i$.
Then, when $|r_2|\leq \frac12$, we have 
\[\mathrm{KL}(Z_1, Z_2)\leq \frac43\lb(r_1 - r_2\rb)^2.\]
\end{lemma}
\proof{Proof.}
By definition, we can write $Z_i=2X_i - 1$ where $Z_i\sim \mathrm{Ber}\lb(\frac{r_i + 1}2\rb)$ for $i=1,2$.
Then, we have
\begin{align*}
\mathrm{KL}(Z_1, Z_2) &=   \mathrm{KL}(X_1, X_2)\\
&=\frac{r_1+1}{2}\ln \frac{r_1+1}{r_2+1} + \frac{1-r_1}{2}\ln \frac{1-r_1}{1-r_2} \\
&\leq \frac{r_1+1}{2}\frac{r_1-r_2}{r_2+1} + \frac{1-r_1}{2}\frac{r_2-r_1}{1-r_2} \\
&\leq \frac{(r_1-r_2)^2}{1-r_2^2}.
\end{align*}
Note that $r_2\le \frac 12$, so the above is bounded by $\frac43(r_1-r_2)^2.$ \hfill$\square$

The following says that two \ins s with small KL divergence are hard to distinguish between.
\begin{theorem}[Pinsker's Inequality]
\label{thm:pinsker}
Let $X,Y\in \{\pm 1\}^n$ be two random vectors.
For any event\footnote{In this work, by ``event'' we mean a Borel set.} $E$, we have \[2(\ho{P}[Y\in E] - \ho{P}[X\in E])^2 \le \mathrm{KL}(X,Y).\]
\end{theorem}

The chain rule \kehua s the KL-divergence for random vectors, on which we will later apply Pinsker's \ineq.
The following can be found as Theorem 2.4 (b) in \citealt{slivkins2019introduction}.

\begin{theorem}[Chain Rule for Product Distributions]\label{thm:chain}
\Sps\ $X_1,\dots, X_n, Y_1,\dots, Y_n\in \{\pm 1\}$ are \indep. 
Consider $X=(X_i)$ and $Y=(Y_i)$.
Then, \[\mathrm{KL}(X,Y)=  \sum_{t=1}^n\mathrm{KL}(X_t,Y_t).\]
\end{theorem}

\noindent{\bf Proof of Lemma \ref{lem:non-optimal}.}
For any color $\chi\in \{\mathrm{r},\mathrm{b}\}$, denote by $Z_\chi=(Z_\chi^s)_{s=1}^{t-1}$ the reward vector under \ins\ $u\oplus \chi$.
Consider the event $E_t :=\{A_t = 0\}$. 
By Pinsker's \ineq\ (Theorem~\ref{thm:pinsker}) and the chain rule (Theorem~\ref{thm:chain}),
\begin{align}\label{eqn:010322}
|\ho{P}[Z_\mathrm{r}\in E_t] - \ho{P}[Z_\mathrm{b}\in E_t]|^2 \leq \frac12\mathrm{KL}(Z_\mathrm{r}, Z_\mathrm{b})  = \frac12 \sum_{s=1}^t \mathrm{KL}(Z_\mathrm{r}^s, Z_\mathrm{b}^s)  = 0+ \frac12 \sum_{s=t_{j-1}}^t \mathrm{KL}(Z_\mathrm{r}^s, Z_\mathrm{b}^s).
\end{align}
By the construction of $\mathcal{F}_\beta$ and 
Lemma~\ref{lem:RC}, for $t_{j-1} < s\leq t$ we have 
$\mathrm{KL}(Z_\mathrm{r}^s, Z_\mathrm{b}^s)\leq  \frac43\lb(\frac 12 C_\beta (2\delta)^{\beta}\rb)^2$, and thus
\begin{align}\label{eqn:011522}
\eqref{eqn:010322} \leq \frac 12 \cdot \frac 43 \lb(\frac 12 C_\beta (2\delta)^{\beta}\rb)^2 \cdot (t - t_{j-1}) \leq \frac {2^{2\beta} C_\beta^2}3 \delta^{2\beta} \cdot 6\delta T,
\end{align}
where the last \ineq\ follows since by definition, $t- t_{j-1} \le t_{j} - t_{j-1} = 6\delta T.$
Finally, recall that $\delta=\lb(2^{2(\beta+1)} C_\beta^2 T\rb)^{-\frac 1{2\beta+1}},$
so (\ref{eqn:011522}) gives  
$\Big|\ho{P}[Z_\mathrm{r} \in E_t]-\ho{P}[Z_\mathrm{b} \in E_t] \Big|^2 \leq \frac 14.$
Therefore,
\begin{align*}
\ho{P}[Z_\mathrm{r} \in E_t] + \ho{P}[Z_\mathrm{b} \in \overline{E_t}]  = \ho{P}[Z_\mathrm{r} \in E_t] + 1 - \ho{P}[Z_\mathrm{b} \in E_t] 
 \geq 1 - \big|\ho{P}[Z_\mathrm{r} \in E_t] - \ho{P}[Z_\mathrm{b} \in E_t]\big|  
 \geq \frac 12,
\end{align*}
i.e., \[\ho{P}_{u\oplus {\rm b}}[A_t = 1] + \ho{P}_{u\oplus {\rm r}}[A_t = 0] \ge \frac 12.\eqno\]
\hfill$\square$

\section{Proof of Theorem \ref{thm:lb}: The Lower Bound} 
\label{adpx:main_lb}
By abuse of notation, for any policy $A$ and binary vector $v$, we write $\mathrm{Reg}(A, v) =\mathrm{Reg}(A, \mu_v)$.
Consider the regret on an epoch $j$ under \ins\ $v$,
\[\textstyle\mathrm{Reg}_j(A, v):=\ho{E}\lb[\sum_{t=t_j}^{t_{j+1}} \lb(r^*_v(t) - Z_{A_t}^t \rb)\rb],\]
where $r^*_v(t)= \max\lb\{0, \mu_v\lb(\frac tT\rb)\rb\}$.
Note that $\mathrm{Reg}_j(A, v)$ depends solely on the first $j$ epochs, so we  only need to specify the first $j$ entries of $v$.
Under this notation,
by Lemma~\ref{lem:non-optimal}, for any prefix $u\in \{\mathrm{r},\mathrm{b}\}^{j-1}$, 
\begin{align*} 
\mathrm{Reg}_j(A, u\oplus \mathrm{r}) + \mathrm{Reg}_j(A, u\oplus \mathrm{b})  
& \geq \sum_{t=t_{j-1} +2\delta T}^{t_{j-1} + 4\delta T} \left(\ho{P}_{u \oplus {\rm b}}[A_t = 1] + \ho{P}_{u \oplus {\rm r}}[A_t = 0] \right) \cdot\delta^\beta &\\
&\quad\geq \ 2\delta T \cdot \frac12 \cdot \delta^\beta= \delta^{\beta+1}T.  
\end{align*}
Thus for some color $v_j \in \{\mathrm{r},\mathrm{b}\}$, we have 
$\mathrm{Reg}_j(A, u\oplus v_j) \geq \frac12\delta^{\beta+1}T$.
Note that this \ineq\ holds for any epoch $j$ and prefix $u\in \{\mathrm{r},\mathrm{b}\}^{j-1}$, so we can inductively construct a \xulie\ $v\in \{\mathrm{r},\mathrm{b}\}^m$ with 
$\mathrm{Reg}_j(A, v[j]) \geq \frac 12 \delta^{\beta+1}T$ for each $j\in [m]$ where $v[j]=(v_1,\dots, v_j)$.
Summing over $j\in [m]$, we conclude that
\[\mathrm{Reg}(A, v) = \sum_{j=1}^m \mathrm{Reg}_j(A, v[j])  \geq m\cdot \frac12\delta^{\beta+1}T 
= \frac{1}{12}\delta^{\beta}T.\]
Substituting $\delta=\lb(2^{2(\beta+1)} C_\beta^2 T\rb)^{-\frac 1{2\beta+1}}$, we conclude that 
\[\mathrm{Reg}(A, v)> \frac 1{24} \cdot 2^{-\beta} \lb(C_\beta\rb)^{-\frac {2\beta}{2\beta+1}}\cdot T^{\frac{\beta+1}{2\beta+1}}.
\eqno\]
\hfill$\square$

\section{Preliminaries for the Upper Bound}
Before presenting the proofs of Theorems~\ref{thm:be_beta1} and \ref{thm:be_beta2}, we first state and prove tools used in both proofs.
We will focus on bounding the regret conditional on the following {\it clean event}, which will be shown to occur with high \prb.
Loosely, this is the event that the rewards in all sufficiently large time intervals obey Hoeffding's \ineq.

\begin{definition}[Clean Event]
For any arm $a$ and rounds $t,t'\in [T]$, consider the event
\[\mathcal{C}_a^{t,t'} = \lb\{\sum_{s=t}^{t'} \lb(Z_{a,s} - \mu_a(s)\rb) \leq \sqrt{6\log T \cdot (t'-t)}\rb\}.\]
We define the {\it clean event} as $\mathcal{C} = \bigcap_{a,t,t'} \mathcal{C}_a^{t,t'}$ where  the intersection is over all arms $a$ and all $t,t'$ with
$t'-t\geq 2\log T$.
\end{definition}

To show the event $\cal C$ occurs with high \prb, we use the following basic concentration bound, whose proof can be found in, e.g., \citealt{vershynin2018high}.

\begin{theorem}[Concentration Bounds]
\label{lem:hoeffding}\label{thm:hoeffding}
Let $Z_1,...,Z_m$ be \indep\ \rv s supported on $[-1,1]$ and $Z = \sum_{i=1}^m Z_i$. 
Then for any $\delta>0$, it holds that
\begin{align*}
\ho{P}(|Z - \ho{E}[Z]| > \delta) &\leq \exp\lb(-\frac {\delta^2}{2m}\rb).
\end{align*}
\end{theorem}

\begin{lemma}[Clean Event Occurs w.h.p.]
\label{lem:clean_event}
For any $T$, it holds that $\ho{P}[\overline{\mathcal{C}}]\leq T^{-1}$.
\end{lemma}
\proof{Proof.} By Hoeffding's \ineq\ \citep[Theorem 2.2.6]{vershynin2018high}, 
for any $1\leq t\leq t'\leq T$ with $t'-t\geq 2\log T$, taking $\delta = \sqrt{6 \log T \cdot (t'-t+1)}$, we have 
\begin{align*}
\ho{P}\lb(\overline{\mathcal{C}_a^{t,t'}}\rb) &\leq \exp\lb(-\frac 1{2(t'-t+1)}\cdot 6\log T \cdot (t'-t+1) \rb) = T^{-3}.
\end{align*}
There are at most $T^2$ combinations of $t,t'$, so by the union bound, we have 
\[\ho{P}[\overline{\mathcal{C}}] =\ho{P}\lb[\bigcup_{a,t,t'} \overline{\mathcal{C}_{a}^{t,t'}}\rb] \le \sum_{a,t,t'}  \ho{P}\lb(\overline{\mathcal{C}_a^{t,t'}}\rb) \leq T^{-1}.\eqno\]
\hfill$\square$

\section{Proof of Proposition \ref{prop:ub_beta=1} (the $T^{2/3}$ Upper Bound for $\beta=1$)}\label{apdx:beta=1}
We first consider a positive epoch $i$. 
In this case, the optimal arm is arm $1$, which coincides with the choice of the BE policy before the epoch's stopping rule is triggered.
Thus, there is no regret in this epoch before time $t_i + S_i$.
This can be formally shown by rephrasing Wald's classical identity as follows. 
Recall that $(Z_1^t)$ are the rewards of arm $1$.


\begin{lemma}[Wald's Identity, Rephrased]\label{lem:wald}  For any epoch $i$, we have 
\[\ho{E}\lb[\sum_{t=t_i}^{t_i + S_i} Z_1^t\rb] =\ho{E}\lb[\sum_{t=t_i}^{t_i + S_i }r(t)\rb].\]
\end{lemma}
\proof{Proof.} For simplicity fix $i$ and  write $S=S_i$. Observe that 
\begin{align}
\label{eqn:011022}    
\ho{E}\lb[\sum_{t=t_i}^{t_i + S} Z_1^t \rb] &= \ho{E}\lb[\sum_{t=t_i}^{t_i + S} \lb(Z_1^t -r(t)\rb)\rb] + \ho{E}\lb[\sum_{t=t_i}^{t_i + S}r(t)\rb].
\end{align}
Note that $Z_1^t- r(t)$'s are \indep, mean $0$, and $S$ is a stopping time, so the partial sum $M_s := \sum_{t=1}^s (Z_1^t - r(t))$ is a \mtg\ (w.r.t. the filtration induced by $\{Z_1^s\}$). 
By Wald's stopping time theorem, $\ho{E}[M_S] = 0$. Therefore, \eqref{eqn:011022} becomes $0+ \ho{E}\lb[\sum_{t=t_i}^{t_i + S}r(t)\rb].$ \hfill$\square$

As a result, we only need to bound the regret incurred after $t_i + S_i$. 
We will do this by bounding the \prb\  that the budget is ever run out, i.e., $S_i < \Delta T$, as detailed in the next lemma. 
Recall that $R_t = \max\{0, r(t)\} - Z_{A_t}^t$ is the regret in round $t$.

\begin{lemma}[Regret on Positive Epochs]\label{lem:no_crossing>0}
\Sps\ $\mu(x) > 0$ for $x \in [x_i, x_{i+1}]$.
Then, whenever $B^2 \ge 6 \Delta T \cdot \log T$, the regret of $\mathrm{BE}(B,\Delta)$ on epoch $i$ \sats\ \[\ho{E}\lb[\sum_{t=t_i}^{t_{i+1}} R_t\rb] \le 1.\]
\end{lemma}
\proof{Proof.}
Write $S=S_i$.
Since $\mu >0$, the optimal policy always chooses arm $1$ in this epoch. 
Recall that at time $t_i +S$ the BE \alg\ switches to arm $0$, the sub-optimal arm. 
We can thus simplify the regret as
\begin{align}\label{eqn:011023}
\ho{E}\lb[\sum_{t=t_i}^{t_{i+1}} \lb(r(t) - Z_{A_t}^t\rb)\rb] & =
\ho{E}\lb[\sum_{t=t_i}^{t_i + S} \lb(r(t) - Z_1^t\rb)  + \sum_{t=t_i +S+1}^{t_{i+1}} \lb(r(t) - Z_0^t\rb)\rb]\notag\\
&= 0 + \ho{E}\lb[\sum_{t=t_i +S+1}^{t_{i+1}} \lb(r(t)-Z_0^t\rb)\rb]\notag\\
&=   \ho{E}\lb[\sum_{t=t_i +S+1}^{t_{i+1}}  r(t) \rb],
\end{align}
where the last identity follows
since $Z_0^t$'s are mean $0$ and \indep\ of $S$. 
Further, since $t_{i+1} - t_i = \Delta T$ and $|r| \le 1$, we have
\begin{align*}
\eqref{eqn:011023} &\le \ho{P}[S = \Delta T -1] \cdot  0 + \ho{P}[S < \Delta T -1] \cdot \Delta T = \ho{P}[S < \Delta T - 1] \cdot \Delta T.
\end{align*}

We conclude the proof by bounding $\ho{P}[S < \Delta T-1]$. 
Consider the event $\{S=s\}$ where  $s < \Delta T-1$.
We claim that the event $\{S=s\}$ would not occur \cond al on the clean event $\cal C$.
In fact, if $\{S=s\}$ occurs,  we have $\sum_{t=1}^s Z_1^t< -B$. 
However, conditional on $\cal C$, we have
\[\lb|\sum_{t=t_i}^{t_i + s} (Z_1^{t} - r(t))\rb| \le \sqrt{6s\cdot \log T} < \sqrt{6\Delta T\log T},\]
and more explicitly, 
\[\sum_{t=t_i}^{t_i + s}  Z_1^t > \sum_{t=t_i}^{t_i +s} r(t) -\sqrt{6\Delta T \log T} \ge 0-B,\] 
where the last \ineq\ follows since $6\Delta T \log T\le B^2$.
\IFT\ $\ho{P}[\{S=s\}\cap \mathcal{C}]=0$ for any $s< \Delta T-1$, and hence
\[\ho{P}[S < \Delta T] 
= \ho{P}\lb[\{S < \Delta T\}\cap \overline{\cal{C}}\rb]
\leq \ho{P}\lb[\overline{\cal{C}}\rb] \leq T^{-1}.\] 
Therefore,
\[(\ref{eqn:011023}) \le \ho{P}[S < \Delta T] \cdot \Delta T\leq T^{-1}\cdot \Delta T \le 1.\eqno\]
\hfill$\square$

The above essentially follows from the definition of the clean event. 
When $\cal C$ occurs, the cumulative reward up to the first $s$ rounds in this epoch lies within an interval of width $w(s) \lesssim \sqrt s$ of the mean, which is positive.\footnote{When illustrating high level ideas, we use $A\lesssim B$ to denote $A = \Tilde O(B)$.}
Further, by the \assu\ that $B\gtrsim \sqrt{\Delta T}$ we have $w(s)\le B$ whenever $s\le \Delta T$.

Now we turn to the negative epochs. The result below follows essentially from the definition of the stopping time $S_i$. 
In fact, if the process never stops until the end of epoch, then the cumulative reward is above $-B$. 
If it does stop at some $s<\Delta T$, then the cumulative reward is {\bf just} below $-B$, and hence above $-(B+1)$ since  $|Z_1^t|\le 1$. 

\begin{lemma}[Regret on Negative Epochs]\label{lem:no_crossing<0}
If $\mu(t) < 0$ for  $x\in [x_i, x_{i+1}]$, then the regret on epoch $i$ \sats\  \[\ho{E}\lb[\sum_{t=t_i}^{t_{i+1}} R_t\rb] \leq B+1.\]
\end{lemma}

\proof{Proof.}
Write $S=S_i$. In this case, the optimal arm is arm $0$, so we can simplify the regret as 
\begin{align}\label{eqn:010123}
\ho{E}\lb[\sum_{t=t_i}^{t_{i+1}} (- Z_{A_t}^t) \rb] = -\ho{E}\lb[\sum_{t=t_i}^{t_i+S} Z_1^t\rb]  - \ho{E}\lb[\sum_{t=t_i + S+1}^{t_{i+1}} Z_0^t\rb].
\end{align}
Note that $Z_0^t$ is \indep\ of $S$, so the second expectation is $0$.

To analyze the first term, for any $s\geq 0$ define $X_s:= - \sum_{t=t_i}^{t_i +s} Z_1^t$. 
Then by definition of $S$, on the event $\{S = s\}$ we have $X_{s-1} < B$. 
Since we assumed each the reward \distr\ to be $\{\pm 1\}$-valued, this implies that $X_s < B + 1$.
Therefore, 
\begin{align*}
-\ho{E}\lb[\sum_{t=t_i}^{t_i +S } Z_1^t\rb] &= \ho{E}[X_S] = \sum_{s=1}^\infty \ho{E}\lb[X_s\cdot \mathbbm{1}(X=s)\rb] < \sum_{s=1}^\infty\ho{P}[S=s]\cdot (B+1)= (B+1),
\end{align*}
where the first identity follows from Lebesgue's Dominated Convergence Theorem and the boundedness of $X_s$ for any fixed $s$.
The claimed bound immediately follows by combining the above with \eqref{eqn:010123}.\hfill$\square$

Finally we consider crossing epochs. The following essentially follows from Lipschitzness of $\mu$.

\begin{lemma}[Regret on Crossing Epochs]\label{lem:regret_crossing_beta=1}
Let $j$ be a crossing epoch, i.e.,  $\mu(\tilde x) =0$ for some $\tilde x\in [x_j, x_{j+1}]$.
Then, the regret in this epoch \sats\ \[\ho{E}\lb[\sum_{t=t_j}^{t=t_{j+1}} R_t\rb] \le 2L \Delta^2 T.\]
\end{lemma}
\proof{Proof.}
By Lipschitzness of $\mu$, we have $|\mu(x)| = |\mu(x)-\mu(\tilde x)| \le  L\Delta$  whenever $x_j\le  x\le x_{j+1}$. 
In the original time scale, this means $|r(t)|\le L\Delta$ whenever $t_j \le t\le t_{j+1}$, and hence 
\[\sum_{t=t_j}^{t_{j+1}} |r(t)| \le \Delta T  \cdot L\Delta = L\Delta^2 T.\]
To connect the above with the regret, observe that
\begin{align*}
\ho{E}\lb[\sum_{t=t_j}^{t_{j+1}} R_t\rb]  = \ho{E} \lb[ \sum_{t=t_j}^{t_{j+1}} \lb(\max\{0, r(t)\} - Y_t\rb)\rb]  \leq \sum_{t=t_j}^{t_{j+1}}  \big|r(t)\big|   - \ho{E}\lb[\sum_{t=t_j}^{t_{j+1}} Z_{A_t}^t\rb].
\end{align*}
Moreover, for each $t$ we have $|\ho{E} [Z_{A_t}^t]| \leq |r(t)|$, and therefore  
\[\ho{E} \lb[\sum_{t=t_j}^{t_{j+1}}  R_t\rb] \le 2 \sum_{t=t_j}^{t_{j+1}}  \big|r(t)\big| \le 2L \Delta^2 T. \eqno\]
\hfill$\square$

We are now ready to formally prove to  bound the regret of the BE \alg\ on $1$-H\"older.

\noindent{\bf Proof of Proposition \ref{prop:ub_beta=1}.}
Let $J_+, J_-, J_{\mathrm{x}} \subseteq \{1,\dots,\Delta^{-1}\}$ be the subsets of positive, negative and crossing epochs. 
Note that  $J_+ \cup J_- \cup J_{\mathrm{x}} = 
\{1,\dots,\Delta^{-1}\}$, so the total regret can be decomposed as 
\begin{align}\label{eqn:011323}
\sum_{i=1}^{1/\Delta} R[i] = \sum_{i\in J_-} R[i]  +  \sum_{i\in J_+} R[i]  +  \sum_{i\in J_\mathrm{x}} R[i].
\end{align}
By Lemma \ref{lem:no_crossing>0}, Lemma \ref{lem:no_crossing<0} and Lemma \ref{lem:regret_crossing_beta=1}, whenever $6\Delta T\log T \le B^2$, we have
\[ \eqref{eqn:011323} < |J_+| \cdot 1 +  |J_-| \cdot (B+1) + |J_\mathrm{x}| \cdot 2 L\Delta^2 T \le \Delta^{-1}\cdot (1 + B + L\Delta^2 T).\eqno\]
\hfill$\square$

\section{Proof of Lemma \ref{lem:regret_crossing_beta=2}: Key Lemma for the $\tilde O(T^{3/5})$ Upper Bound} \label{apdx:regret_crossing_beta=2}
Without loss of generality, we assume $\ell \ge 0$ 
As the key observation, we first claim that $|\mu'| = O\lb((\ell+1) \Delta\rb)$ on epoch $j$. 
In fact, since epoch $(j-\ell)$ is stationary, there exists  an  $s\in [x_{j-\ell}, x_{j-\ell+1}]\subseteq [0,1]$ with $\mu'(s) =0$.
Moreover, since $\mu'$ is Lipschitz, for any $x\in [x_j, x_{j+1}]$ we have 
\[|\mu'(x)| = |\mu'(x) - \mu'(s)|  \le L\cdot |x-s|\leq L \cdot (\ell+1) \Delta.\]

We next claim that $|r(t)| = O((\ell+1)\Delta^2)$ on epoch $j$.
In fact, let $\tilde x\in [x_j, x_{j+1}]$ be any crossing, i.e., $\mu(\tilde x)=0$.
Then for any $x\in [x_j, x_{j+1}]$, by the mean value theorem, for some $\zeta$ between  $\tilde x$ and $x$, i.e., $(\zeta-x)\cdot (\zeta - \tilde x) \le 0$, it holds that 
\[|\mu(x)| = |\mu(x)-\mu(\tilde  x)|= |\mu'(\zeta) \cdot (\tilde  x - x)|.\]
By the previous claim, i.e., $|\mu'|\leq L\cdot (\ell+1)\Delta$ on epoch $j$, the above implies that 
\begin{align*}
|\mu(x)| \leq L\cdot (\ell+1) \Delta\cdot |\tilde x -x|  \leq L\cdot (\ell+1)\Delta^2.
\end{align*}
Translating this to  the original time scale, we have $|r(t)| \le L\cdot (\ell+1)\Delta^2$ for any $t\in [t_j, t_{j+1}]$,  and thus the claim holds.

Finally, summing over all rounds from $t_j$ to $t_{j+1}$, we obtain 
\begin{align}
\label{eqn:121522}
\sum_{t=t_j}^{t_{j+1}} |r(t)| \leq L(\ell+1)\Delta^2 \cdot \Delta T.
\end{align}
We use this to bound  the regret.
Note that
\begin{align}\label{eqn:011523b}
\ho{E}\lb[\sum_{t=t_j}^{t_{j+1}} R_t\rb] &= \ho{E} \lb[ \sum_{t=t_j}^{t_{j+1}} \lb(\max\{0, \mu(t)\} - Y_t\rb)\rb]\notag \\
&\leq \sum_{t=t_j}^{t_{j+1}}  \big|\mu(t)\big| - \ho{E}\lb[\sum_{t=t_j}^{t_{j+1}} Y_t\rb].
\end{align}
Observe that for each $t$ we have $|\ho{E} Y_t| \leq |r(t)|$, so the above is bounded by $2 \sum_{t=t_j}^{t_{j+1}}  \big|r(t)\big|$. 
Combining this with \eqref{eqn:121522}, we conclude that 
\[\eqref{eqn:011523b} \leq 2L \cdot (\ell+1)\cdot \Delta^3 T.\eqno\]
\hfill$\square$

\section{Proof of Proposition \ref{lem:corner}: When There Is No Stationary Point}\label{apdx:corner_case} 
Note that in this case $\mu$ has at most one crossing on its domain $[0,1]$. 
As the trivial case, \sps\ there is no crossing, then the upper bound follows immediately by applying Lemma \ref{lem:no_crossing>0} or  Lemma \ref{lem:no_crossing<0} on each epoch.

Now \sps\ there is exactly one crossing, say $\tilde x\in [x_{i_0}, x_{i_0+1}]$ for some integer $i_0$.
Then, for any $x\in [x_{i_0}, x_{i_0+1}]$, by Lipschitzness, we have \[|\mu(x)| = |\mu(x)-\mu(\tilde x)|= |\mu'(\zeta) \cdot (x-\tilde x)|\le L\Delta.\]
Translating this to the original time scale,  we have $|r(t)| \le L \Delta$ whenever $t_{i_0} \le t\le t_{i_0+1}$.
Therefore, 
\begin{align}
\label{eqn:121522a}
\sum_{t=t_{i_0}}^{t_{i_0+1}} |r(t)| \leq L\Delta \cdot (t_{i_0+1}-t_{i_0})
= L\Delta  \cdot \Delta T.
\end{align}

Meanwhile, since any epoch $i\neq i_0$ is either negative or positive, by Lemma \ref{lem:no_crossing>0} and Lemma \ref{lem:no_crossing<0} we have $R[i] \le B+1$.
Combining this with (\ref{eqn:121522a}), the total regret is then bounded as 
\[\sum_{i=1}^{1/\Delta} R[i] = R[i_0] + \sum_{i\neq i_0} R[i]< L\Delta^2 T + \Delta^{-1}\cdot (B+1).\eqno\]
\hfill$\square$

\section{Proof of Proposition \ref{prop:ub_beta=2}: An $\tilde O(T^{3/5})$ Regret for $\beta=2$, Two-armed Case}
To suppress the notation, we will subsequently write  
$R[i] :=\ho{E}\lb[\sum_{t=t_i}^{t_{i+1}} R_t\rb]$ as the regret on epoch $i$. If there is no stationary point, i.e., $\mu'(t)\neq 0$ for all $t\in [0,1]$, then the desired bound follows immediately from  Proposition \ref{lem:corner}.
\Ow, index the stationary epochs as $s_1< \dots< s_n$.
Observe that for any $j$,  there is at most one crossing epoch between $s_j$ and $s_{j+1}$, say $i_{\mathrm{x}}$, if it does exist. 
Then by Lemma \ref{lem:regret_crossing_beta=2}, 
\[R[i_{\mathrm{x}}] \le 
2L\cdot (|i_{\mathrm{x}}-s_j|+1)\cdot \Delta^3 T \le 2L\cdot (|s_{j+1}-s_j|)\cdot \Delta^3 T.\]
Combining the above with Lemma \ref{lem:no_crossing>0} and Lemma \ref{lem:no_crossing<0}, the total expected regret on epochs $s_j$ through $s_{j+1}$ \sats\ 
\[\sum_{s_j \le i < s_{j+1}}   R[i]\le (s_{j+1} - s_i) \cdot (2L \Delta^3 T + B+1).\]
The above clearly also holds when there is no crossing between $s_j$ and $s_{j+1}$. 
In fact, in this case, the $2L\Delta^3 T$ term disappears.
Summing over $j=1,\dots,n$, we have
\begin{align*}
\textstyle\sum_{i=1}^{1/\Delta} R[i] &\textstyle\le \sum_{j=1}^n \sum_{s_j \le i < s_{j+1}}  R[i] \notag\\
&\textstyle\le \sum_{i=1}^n  (s_{j+1} - s_i)\cdot (2L \Delta^3 T +B + 1)\notag\\
&\textstyle\le 2n\cdot (L\Delta^3 T + B),
\end{align*}
and the desired bound follows by noticing that $n\le \frac 1\Delta$.
\hfill$\square$

\section{Proof of Proposition \ref{prop:k-arm}: An $\tilde O(T^{3/5})$ Bound for the Multi-armed Case}\label{apdx:k_arm}
We provide a formal proof of the upper bound for the general finite-armed setting with $\beta=2$. 

\subsection{Preliminaries}\label{sec:prelim_k_arms}
To formalize the ideas outlined in Section \ref{sec:k_arm}, we extend the notion of {\em crossing} by redefining it as points at which the best arm changes.
We will specifically pay attention to crossings where the overtaker's speed is substantially faster than that of the arm being overtaken, formally defined as follows. 

\bdefn[Crossing and Fast Crossing] 
We say that an arm $\alpha$ {\em overtakes} another arm $\beta$, if there exists $\eps_0 >0$ such that for any $\eps \in [-\eps_0,\eps_0]$, we have $\mu_\alpha(x+\eps)> \mu_\beta(x+\eps)$ if and only if $\eps>0$.
For any $x\in [0,1]$, denote by $a^*(x) \in [k]$ the arm with the highest mean reward at $x$ (where we break ties arbitrarily). 
We say that $x$ is a {\em crossing}, if an arm $a$ overtakes $a^*(x)$ at $x$.
Furthermore, $x$ is a {\em fast-crossing} if $\mu'_a(x) - \mu'_{a^*(x)}(x)\ge \Delta.$\edefn

Our regret analysis employs a {\em potential \func\ argument}, which is quite common in competitive analysis for online \alg s. 
The adversary starts with a certain amount of {\em potential}.
We will show in Section \ref{apdx:non_crossing_general_k} that the regret is low on a (i) {\em non-crossing} epoch, i.e., an epochs that contains no crossing, and (ii) {\em stationary crossing} epoch, i.e., an epoch that contains a crossing but not a fast crossing.

Therefore, to generate a high regret on an epoch, the adversary needs to forge a fast crossing.
As the key step in the analysis, we argue that this consumes a high energy (formally defined in Section \ref{apdx:regret_fast_crossing_k_arm}) and results in a large reduction of the potential. Consequently, the adversary has less power in the future.
Subsequently, we will fix and suppress the epoch length $\Delta\in[0,1]$.

For the analysis, we modify the definition of {\it clean event} as follows.

\begin{definition}[Clean Event]
For any arm $a$ and rounds $t,t'\in [T]$, consider the event
\[\mathcal{C}_a^{t,t'} = \lb\{\sum_{s=t}^{t'} \lb(Z_a^s - \mu_a(s)\rb) \leq \sqrt{\frac{6\log k}k \cdot \log T \cdot (t'-t)} \rb\}.\]
We define the {\it clean event} as $\mathcal{C} = \bigcap_{a,t,t'} \mathcal{C}_a^{t,t'}$ where the intersection is over all arms $a$ and pairs of $t,t'$ with $t'> t+ 2\log T$.
\end{definition}

Compared to the one-armed case, in the above definition, we have an extra $\frac{\log k}k$ factor. 
Intuitively, the term $\log k$ is introduced for applying the union bound on the $k$ arms, and the term $1/k$ occurs since each arm is only guaranteed to be selected $t/k$ times after $t$ rounds. 
Next, we show that $\cal C$ occurs with high \prb.

\begin{lemma}[Clean Event Occurs w.h.p.]
For any $k,T\ge 1$, we have $\ho{P}[\overline{\mathcal{C}}]\leq k^{-1} T^{-2}.$
\end{lemma}
\proof{Proof.}
By Hoeffding's \ineq\ \citep[Theorem 2.2.6]{vershynin2018high}, 
for any $1\leq t\leq t'\leq T$ with $t'-t\geq 2\log T$, taking \[\delta = \sqrt{\frac{6\log k}k \log T \cdot (t'-t)},\] we have 
\[\ho{P}\lb(\overline{\mathcal{C}_a^{t,t'}}\rb) \le \exp\lb(-\frac 1{2(t'-t)/k}\cdot 6\log T \cdot \frac{t'-t}k \rb) = k^{-1} T^{-3}.\]
There are at most $T^2$ combinations of $t,t'$, so by the union bound, we conclude that 
\[\ho{P}[\overline{\mathcal{C}}] =\ho{P}\lb[\bigcup_{a,t,t'} \overline{\mathcal{C}_{a}^{t,t'}}\rb] \le \sum_{a,t,t'}  \ho{P}\lb(\overline{\mathcal{C}_a^{t,t'}}\rb) \le k^{-1} T^{-1}.\eqno\]
\hfill$\square$

\subsection{Non-crossing epochs}\label{apdx:non_crossing_general_k}
In this section bound the regret on non-crossing epochs.

\begin{lemma}[Regret on a Non-crossing Epoch]\label{lem:nx_k_arm}
\Sps\ that $R$ the regret of the BE policy ${\rm BE}(B,\Delta)$ on a non-crossing epoch and $B^2 \ge  k^{-1} \Delta T \log T \log k$, then \[\ho{E}[R] \le 1+kB.\]
\end{lemma}
\proof{Proof.} By definition of $C$, the cumulative reward after $s$ rounds (counting from the start of the epoch) deviates from the mean by at most $\sqrt{k^{-1} s \log T \log k}$. 
Since $s\le \Delta T$, this quantity is at most $\sqrt{k^{-1} \Delta T \log T \log k}$.
Thus, if $B\ge \sqrt {k^{-1} \Delta T \log T \log k}$, the BE policy will never eliminate the best arm. 
\IOW, if all but one arm are eliminated at some time, then the future regret in this epoch is $0$. 
On the other hand, regardless of where $C$ occurs, the regret incurred by each arm is bounded by $(B+1)$ by Lemma \ref{lem:no_crossing<0}. 
Therefore, we can bound the regret on this epoch as 
\[\ho{E}[R] \le \ho{E}[R| \bar C] \cdot \ho{P}[\bar C] + \ho{E}[R| C]\cdot \ho{P}[C] \le \Delta T \cdot T^{-3} + (k-1)\cdot(B+1)\le 1 + kB.\eqno\]
\hfill$\square$

\subsection{Regret on Stationary Epochs}
We next bound the regret on stationary epochs. 
Recall that a crossing epoch is stationary if it does not contain a fast crossing. 
We will categorize the arms into two groups: those that have {\em ever} been the best arm at some point during this epoch and the rest. 
We formalize this notion as follows. Recall that $r_a(t) = \mu_a(t/T)$ for all $t\in [T]$, where $\mu_a:[0,1]\rar [-1,1]$ is a $2$-H\"older \func.

\bdefn[Arms That Have Ever Been the Best] For any epoch $[x_i,x_{i+1}]\sse [0,1]$, we define \[A^*_i = \lb\{a\in [k]:\quad \exists x\in [x_i,x_{i+1}] \text{ such that } \mu_a(x) \ge \mu_{a'}(x), \ \forall a'\in [k]\rb\}\]
\edefn

We first show that, on any  stationary epoch $i$, the mean rewards of the arms in $A^*=A^*_i$ do not differ much.
To see this, observe that, by definition, every $a\in A^*$ ``overtakes'' the best arm at some point. 
Moreover, when this happens, the difference in the derivatives differs by at most $L\Delta$. 
Since there can be at most $k$ crossings in an epoch, the derivatives of any two arms can only differ by $kL\Delta$ on this epoch.
We summarize this observation as the following lemma.

\begin{lemma}[Arms in $A^*$ Have Similar Rewards]\label{lem:A^*} 
Let $[x_i,x_{i+1}]$ be a stationary epoch. 
Then for any $a,b\in A^*_i$ and any $x\in [x_i, x_{i+1}]$, we have \[|\mu_a (x)- \mu_b (x)| \le kL\Delta^2.\]
\end{lemma}

We will show that on the clean event, if all but one arm are eliminated at some point, then the remaining arm must be in $A^*$. 
Consequently, the regret in the exploitation phase is low.
Formally, we have the following bound.
For convenience, we switch to the non-normalized time scale, where we recall that $t_i = x_i  \Delta$.

\begin{lemma}[Regret on Stationary Crossing Epochs] \label{lem:stationary_epoch_k_armed}
\Sps\ $B^2 \ge  k^{-1} \Delta T \log T \log k$.
Let $[t_i,t_{i+1}]\sse [0,T]$ be a stationary epoch. Then, 
\[\ho{E}\lb[\sum_{t=t_i}^{t_{i+1}} R_t \rb]\le kB + kL\Delta^3 T.\]
\end{lemma}
\proof{Proof.} Let $S$ be the time when all but one arms are eliminated and write $A^*=A^*_i$.  
First, we observe that on the clean event $C$, any arm in $A^*$ can never be eliminated for ``losing'' to an arm in $(A^*)^c$. 
In fact, this is because the width of the \ci\ scales as $\sqrt t$, and we assumed that $B^2 \ge  k^{-1} \Delta T \log T \log k$.

We use the above to bound the regret after time $S$. 
By Lemma \ref{lem:A^*}, for any $a,b\in A^*$ we have $|\mu_a(x) - \mu_b(x)| \le kL\Delta^2$ for all $x$. 
In \parti, at any time $x$, let $\mu^*(x) = \max_{a\in [k]} \mu_a(x)$, then 
\[|\mu_a(x) - \mu^*(x)|\le kL\Delta^2.\]
Thus, the regret after $S$ can be bounded as 
\[ \sum_{t=S}^{t_{i+1}} R_t \le L\Delta^2 \cdot \Delta T = kL\Delta^3 T.\]

On the other hand, the regret before $S$ is bounded by $kB$. 
Therefore, the total regret on this epoch is bounded as  \[\ho{E} \lb[ \sum_{t=t_i}^{t_{i+1}} R_t\rb] = \ho{E} \lb[ \sum_{t=t_i}^{S} R_t \rb]+  \ho{E} \lb[ \sum_{t=S}^{t_i + \Delta^{-1} T} R_t \rb] \le kB + kL\Delta^3 T.\eqno\]
\hfill$\square$

\subsection{Regret on Fast Crossing Epochs}\label{apdx:regret_fast_crossing_k_arm}
We next show that each epoch has at most one fast crossing for each ordered pair of arms.

\begin{lemma}[At Most One Fast Crossing Per Epoch]\label{lem:at_most_one_fast_crossing}
For any pair $\alpha,\beta$ of arms and any epoch, there is at most one fast crossing where $\alpha$ overtakes $\beta$.
\end{lemma}
\proof{Proof.} 
\Sps\ $\alpha$ overtakes $\beta$ at some $x_0\in [0,1]$ and $\mu_\alpha'(x) - \mu_\beta'(x) \ge L \Delta$. Denote by $g(x): =\mu_\alpha(x_0+x) - \mu_\beta(x_0+x).$
Then, $g(0)=0$ and $g'(0)\ge L \Delta$. 
By the defintion of H\"older class, $g'$ is $L$-Lipschitz, so for any $x>0$ we have \[g'(x) \ge g'(0) - Lx \ge L\Delta - Lx.\]  
Thus, for any $x>0$, it holds that \[g(x) = \int_0^x g'(s)\ ds \ge \int_0^x L\Delta - Ls \ ds = L\Delta x - \frac L2 x^2.\]
Therefore, for any $0\le x \le 2\Delta$, we have 
\[g(x)  = \mu_\alpha(x) -  \mu_\beta(x) \ge 0.\] 
\IOW, the next time $\alpha$ overtakes $\beta$ occurs at least $\Delta/L$ time later. Therefore, there is at most one fast crossing where $\alpha$ overtakes $\beta$.\hfill$\square$

Note that the above is consistent with the definition for $k=1$. We next introduce the notion of energy, which captures the adversary's ability to cause high regret in the neighborhood of a crossing. 

\bdefn[Energy]\label{def:energy_k=2}
\Sps\ $x_1,\dots, x_m$ are all fast crossings, and at $x_i$, arm $\alpha(i)$ overtakes $\beta(i)$. 
Define the {\em energy} of the $i$-th fast crossing by  \[\kappa_i := |\mu'_{\alpha(i)} (x_i) - \mu'_{\beta(i)} (x_i)|.\]
\edefn

We now show that the total potential is bounded by a constant (in $T$).

\begin{lemma}[Conservation of Energy]\label{lem:fast_crossing}
Under the notations in Definition \ref{def:energy_k=2}, we have 
\[ \sum_{i=1}^m \kappa_i \le (k-1)kL.\]
\end{lemma}
\proof{Proof.} For a fixed (ordered) pair of arms $a,b$, denote by  $x_{i_1},\dots,x_{i_\ell}$ the subsequence of fast crossings
where $a$ overtakes $b$. 
For simplicity, we write $x_j = x_{i_j}$.
Observe that between any two crossings $x_i,x_{i+1}$, there exists at least one stationary point, which we denote by $s_i$. 
By Lipschitzness of $\mu'$, we have  
\[\kappa_i = |\mu'(x_i) - \mu'(s_i)|\le L |s_i - x_i|.\] Therefore, 
\[\sum_{i=1}^\ell \kappa_i \le \sum_{i=1}^\ell L|s_i - x_i| \le L,\]
where the last \ineq\ follows since the intervals $\{(x_i,s_i): i\in [\ell]\}$ are disjoint. 
Since there are $k(k-1)$ distinct ordered pairs of arms, the claimed \ineq\ follows.
\hfill$\square$

Finally, we connect the total energy to the regret. Its proof is similar to that of Lemma \ref{lem:regret_crossing_beta=2}, so we omit it here.

\begin{lemma}[Regret is Proportional to the Energy]\label{lem:regret_energy}
Under the notations in Definition \ref{def:energy_k=2}, denote by $R[i]$ the expected regret of the policy ${\rm BE}(B, \Delta)$ on the epoch containing $x_i$. Then,
\[\ho{E} \lb[ \sum_{i=1}^m  R[i] \rb]\le \sum_{i=1}^m \kappa_i \Delta^2 T.\]
\end{lemma}

Again, we emphasize by Lemma \ref{lem:at_most_one_fast_crossing}, each epoch contains only one fast crossing, and thus $R[i]$ and $R[j]$ correspond to distinct epochs if $i\neq j$. 
We now combine all this and prove the main upper bound. 

\subsection{Proof of Proposition \ref{prop:k-arm}}
Denote by $E_{\rm nx}, E_{\rm st}, E_{\rm fx} \sse \{1,\dots,\Delta^{-1}\}$ the subset of non-crossing epochs, stationary crossing epochs and non-stationary crossing epochs (recall that such epochs are exactly crossing epochs with a fast crossing).
From Lemma \ref{lem:nx_k_arm}, the total regret on each in $E_{\rm nx}$ is at most $1+kB$.
By Lemma \ref{lem:stationary_epoch_k_armed}, the regret on each epoch in $E_{\rm st}$ is bounded by $kB + kL \Delta^3 T$. 
Therefore, \[\sum_{j \in E_{\rm st}} R[j] \le |E_{\rm st}| \cdot (kB + kL \Delta^3 T).\]
Combining Lemma \ref{lem:fast_crossing} and Lemma \ref{lem:regret_energy}, we have 
\[\ho{E} \lb[ \sum_{i\in E_{\rm fx}}  R[i] \rb] \le  \sum_{i=1}^m \kappa_i \Delta^2 T < k^2 L \Delta^2 T.\]
Therefore, we conclude that 
\begin{align*}
\ho{E}\lb[\sum_{i=1}^{1/\Delta} R[i]\rb] &= \ho{E}\lb[\sum_{i\in E_{\rm nx}}R[i] + \sum_{i\in E_{\rm st}} R[i] +\sum_{i\in E_{\rm fx}}R[i]\rb] \\
&\le \Delta^{-1} (1+kB) + |E_{\rm st}| \cdot (kB + kL \Delta^3 T) + k^2 L \Delta^2 T\\
&\le 2\Delta^{-1} (1+kB) + 2L k^2 \Delta^2 T,
\end{align*}
where the last \ineq\ follows since $|E_{\rm st}|\le 1/\Delta$.\hfill$\square$
\end{document}